\documentclass{article}

\usepackage{microtype}
\usepackage{graphicx}
\usepackage{subcaption}
\usepackage{booktabs} 

\usepackage{hyperref}

\usepackage[accepted]{icml2023}

\usepackage{amsmath}
\usepackage{amssymb}
\usepackage{amsfonts}
\usepackage{mathtools}
\usepackage{amsthm}

\usepackage[capitalize,noabbrev]{cleveref}

\theoremstyle{plain}
\newtheorem{theorem}{Theorem}[section]

\theoremstyle{definition}

\theoremstyle{remark}

\DeclareMathOperator*{\argmax}{arg\,max}

\usepackage{mathabx}

\usepackage[textsize=tiny]{todonotes}

\icmltitlerunning{Optimizing DDPM Sampling with Shortcut Fine-Tuning}

\begin{document}

\twocolumn[
\icmltitle{Optimizing DDPM Sampling with Shortcut Fine-Tuning}



\icmlsetsymbol{equal}{*}

\begin{icmlauthorlist}
\icmlauthor{Ying Fan}{yyy}
\icmlauthor{Kangwook Lee}{yyy}
\end{icmlauthorlist}

\icmlaffiliation{yyy}{UW Madison}

\icmlcorrespondingauthor{Ying Fan, Kangwook Lee}{yfan87@wisc.edu, kangwook.lee@wisc.edu}

\icmlkeywords{Machine Learning, ICML}

\vskip 0.3in
]



\printAffiliationsAndNotice{} 

\begin{abstract}
In this study, we propose \emph{Shortcut Fine-Tuning (SFT)}, a new approach for addressing the challenge of fast sampling of pretrained Denoising Diffusion Probabilistic Models (DDPMs). SFT advocates for the fine-tuning of DDPM samplers through the direct minimization of Integral Probability Metrics (IPM), instead of learning the backward diffusion process. This enables samplers to discover an alternative and more efficient sampling shortcut, deviating from the backward diffusion process. \textcolor{black}{Inspired by a control perspective,} we propose a new algorithm 
\textbf{SFT-
PG}: \textbf{S}hortcut \textbf{F}ine-\textbf{T}uning with \textbf{P}olicy \textbf{G}radient, and prove that under certain assumptions, gradient descent of diffusion models with respect to IPM is equivalent to performing policy gradient. \textcolor{black}{To our best knowledge, this is the first attempt to utilize reinforcement learning (RL) methods to train diffusion models.} Through empirical evaluation, we demonstrate that our fine-tuning method can further enhance existing fast DDPM samplers, resulting in sample quality comparable to or even surpassing that of the full-step model across various datasets.

\end{abstract}

\section{Introduction}

Denoising diffusion probabilistic models (DDPMs)~\citep{ho2020denoising} are parameterized stochastic Markov chains with Gaussian noises, which are learned by gradually adding noises to the data as the forward process, computing the posterior as a backward process, and then training the DDPM to match the backward process. Advances in DDPM~\citep{nichol2021improved,dhariwal2021diffusion} have shown the potential to rival GANs~\citep{gan} in generative tasks. However, one major drawback of DDPM is that a large number of steps $T$ is needed. As a result, there is a line of work focusing on sampling fewer $T'\ll T$ steps to obtain comparable sample quality: Most works are dedicated to better approximating the backward process as stochastic differential equations (SDEs) with fewer steps, generally via better noise estimation or computing better sub-sampling schedules~\citep{kong2021fast, san2021noise, lam2021bilateral, watson2021learningto, jolicoeur2021gotta, bao2021analytic, bao2022estimating}. Other works aim at approximating the backward process with fewer steps via more complicated non-gaussian noise distributions~\citep{xiao2021tackling}.\footnote{There is another line of work focusing on fast sampling of DDIM~\citep{song2020denoising} with deterministic Markov sampling chains, which we will discuss in Section~\ref{sec: related}.}

\begin{figure}[t!]
\centering
\includegraphics[width=0.9\columnwidth]{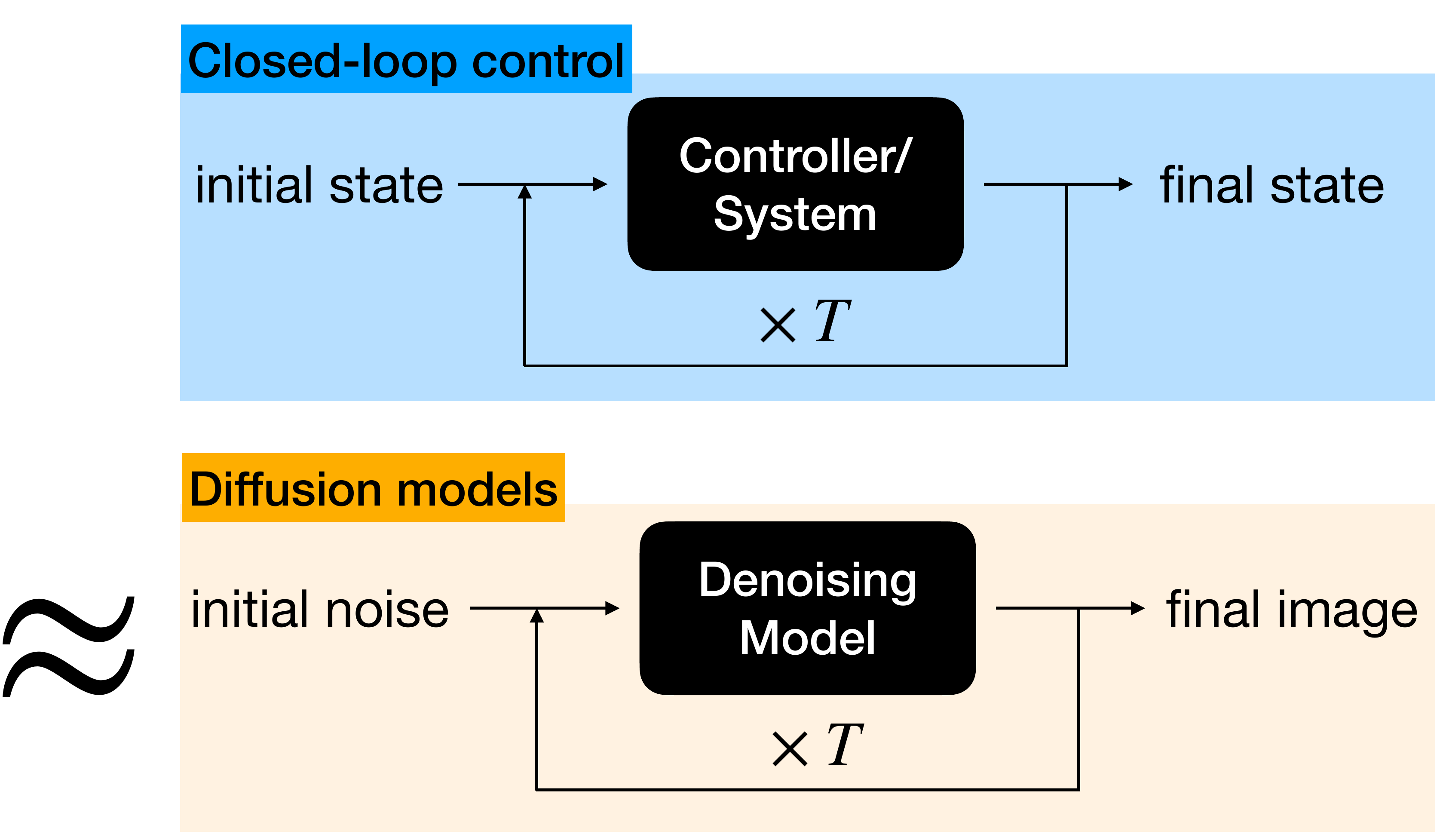}
\caption{\textcolor{black}{Image denoising is similar to a closed-loop control system: finding paths from pure noise to natural images.}}
\label{fig:control}
\end{figure}

To our best knowledge,  existing fast samplers of DDPM stick to imitating the computed backward process with fewer steps. If we treat data generation as a \textcolor{black}{control task (see Fig.~\ref{fig:control}), the backward process can be viewed as a demonstration to generate data from noise (which might not be optimal in terms of number of steps), and the training dataset could be an environment that provides feedback on how good the generated distribution is. From this view, imitating the backward process could be viewed as imitation learning~\citep{hussein2017imitation} or behavior cloning~\citep{torabi2018behavioral}. Naturally, one may wonder if we can do better than pure imitation, since learning via imitation is generally useful but rarely optimal, and we can explore alternative paths for optimal solutions during online optimization.
}

\begin{figure}[t!]
\centering
\includegraphics[width=0.9\columnwidth]{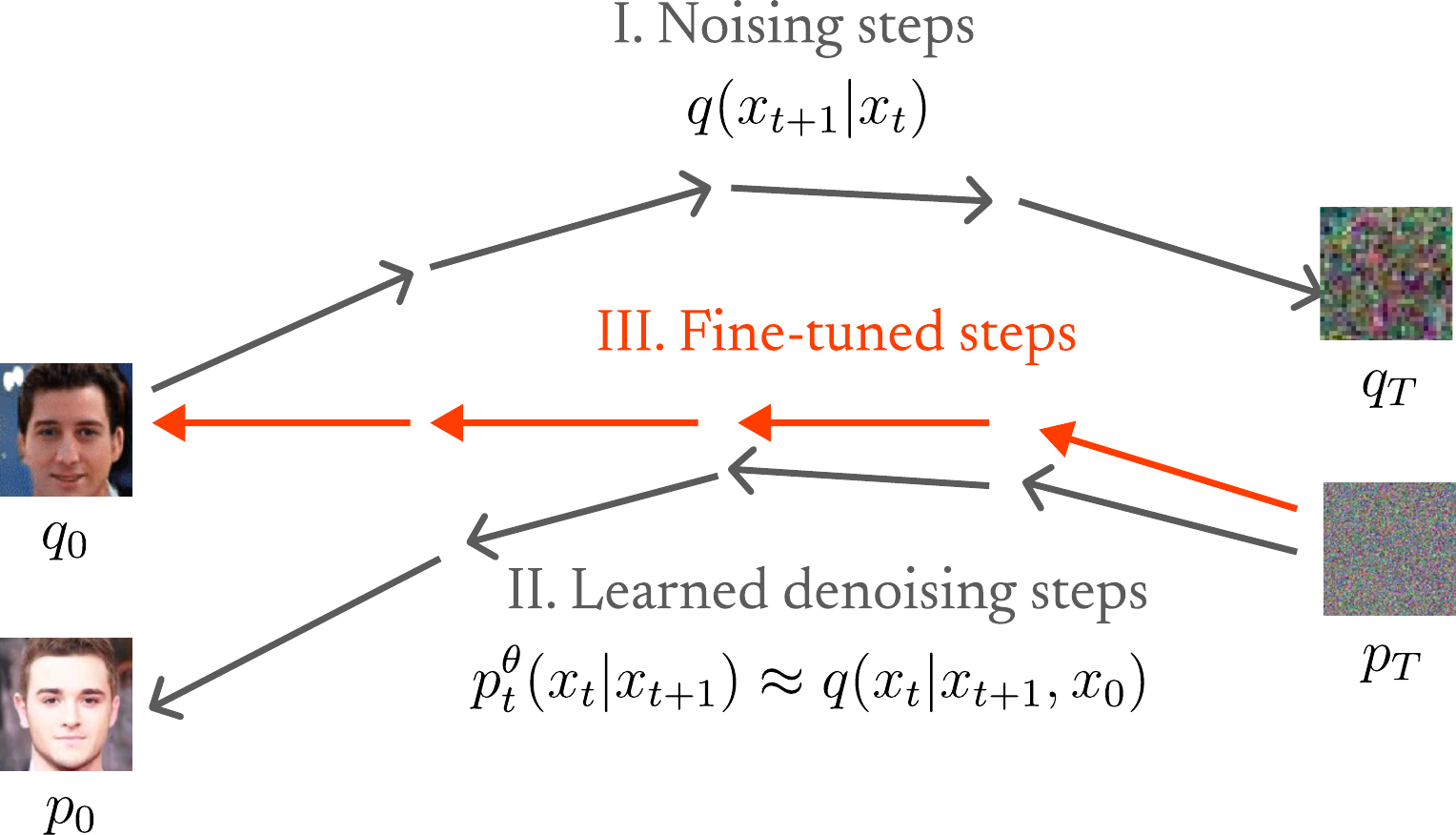}
\caption{
A visual illustration of the key idea of Shortcut Fine-Tuning (SFT).
DDPMs aim at learning the backward diffusion model, but this approach is limited to a small number of steps.
We propose the idea of \emph{not} following the backward process and exploring other unexplored paths that can lead to improved data generation.
To this end, we directly minimize an IPM and develop a policy gradient-like optimization algorithm. 
Our experimental results show that one can significantly improve data generation quality by fine-tuning a pretrained DDPM model with SFT. We also provide a visualization of the difference between steps in II and III when $T$ is small in Appendix~\ref{app: vis}.
}
\label{fig:smallt}
\end{figure}

Motivated by the above observation, we study the following underexplored question: 
\vspace{-0.1in}
\begin{center}
\emph{Can we improve DDPM sampling by \textbf{not} following the backward process?}
\end{center}



In this work, we show that this is indeed possible.
We fine-tune pretrained DDPM samplers by directly minimizing an integral probability metric (IPM) and show that finetuned DDPM samplers have significantly better generation qualities when the number of sampling steps is small. 
In this way, we can still enjoy diffusion models' multistep capabilities with no need to change the noise distribution, and improve the performance with fewer sampling steps.

More concretely, we first show that performing gradient descent of the DDPM sampler w.r.t. the IPM is equivalent to stochastic policy gradient, which echoes the aforementioned RL view but with a changing reward from the optimal critic function given by IPM. In addition, we present a surrogate function that can provide insights for monotonic improvements. Finally, we present a fine-tuning algorithm with alternative updates between the critic and the generator.

We summarize our main contributions as follows: 

\begin{itemize}
    \item (Section~\ref{sec: pg}) We propose a novel algorithm to fine-tune DDPM samplers with direct IPM minimization, and we show that performing gradient descent of diffusion models w.r.t. IPM is equivalent to policy gradient. To our best knowledge, this is the first work to apply reinforcement learning methods to diffusion models.

    \item (Section~\ref{sec: mono}) We present a surrogate function of IPM in theory, which provides insights on conditions for monotonic improvement and algorithm design.
    
    \item (Section~\ref{sec: regularization}) We propose a regularization for the critic based on the baseline function, which shows benefits for the policy gradient training.


    \item (Section~\ref{sec: full exp}) Empirically, we show that our fine-tuning can improve DDPM sampling performance in two cases: when $T$ itself is small, and when $T$ is large but using a fast sampler where $T' \ll T$. In both cases, our fine-tuning achieves comparable or even higher sample quality than the DDPM with 1000 steps using 10 sampling steps.

\end{itemize}

\section{Background}

\subsection{Denoising Diffusion Probabilistic Models (DDPM)}
Here we consider denoising probabilistic diffusion models (DDPM)  as stochastic Markov chains with Gaussian noises ~\citep{ho2020denoising}. 
Consider data distribution $x_0 \sim q_0, x_0 \in \mathbb{R}^{n}$. 

Define the forward noising process: for $t \in [0,..,T-1]$,
\begin{equation}
   q(x_{t+1}|x_{t}) := \mathcal{N}(\sqrt{1-\beta_{t+1}}x_{t}, \beta_{t+1}I), 
\end{equation}
where $x_1,.., x_T$ are variables of the same dimensionality as $x_0$, $\beta_{1:T}$ is the variance schedule. 

We can compute the posterior as  a backward process:
\begin{equation}
    q(x_{t}|x_{t+1}, x_0) = \mathcal{N}(\tilde{\mu}_{t+1}(x_{t+1}, x_0), \tilde{\beta}_{t+1}I),
\end{equation}
where $\tilde{\mu}_{t+1}(x_{t+1}, x_0) = \frac{\sqrt{\bar{\alpha}_t}\beta_t}{1-\bar{\alpha}_{t+1}}x_0 + \frac{\sqrt{\alpha_{t+1}}(1-\bar{\alpha}_{t})}{1-\bar{\alpha}_{t+1}}x_{t+1}$, $\alpha_{t+1} = 1-\beta_{t+1}$, $ \bar{\alpha}_{t+1} = \prod_{s=1}^{t+1}\alpha_{s}$.

We define a DDPM sampler parameterized by $\theta$, which generates data starting from some pure noise $x_T \sim p_T$:
\begin{equation}
\begin{split}
&x_T \sim p_T = \mathcal{N}(0, I),\\
&x_t \sim p_t^{\theta}(x_t|x_{t+1}),\\
&p_t^{\theta}(x_t|x_{t+1}):=\mathcal{N}\big(\mu_{t+1}^\theta(x_{t+1}), \Sigma_{t+1}\big),\\
\end{split}
\end{equation}
where $\Sigma_{t+1}$ is generally chosen as $\beta_{t+1}I$ or $\tilde{\beta}_{t+1}I$. \footnote{In this work we consider a DDPM sampler with a fixed variance schedule $\beta_{1:T}$ as in \citet{ho2020denoising}, but it could also be learned as in \citet{nichol2021improved}.}

Define
\begin{equation}
   p^{\theta}_{x_{0:T}} := p_T(x_T)\prod_{t=0}^{T-1}p_t^{\theta}(x_t|x_{t+1}),
\end{equation}
and we have the marginal distribution $p^{\theta}_0(x_0) = \int p_{x_{0:T}}^{\theta}(x_{0:T}) d x_{1:T}$.

The sampler is trained by minimizing the sum of KL divergences for each step: 
\begin{equation}
J = \mathbb{E}_{q}\left[\sum_{t=0}^{T-1}D_{KL}(q(x_{t}|x_{t+1}, x_0), p_t^{\theta}(x_{t}|x_{t+1}))\right].
\end{equation}
Optimizing the above loss can be viewed as matching the conditional generator $p_t^{\theta}(x_t|x_{t+1})$ with the backward process $q(x_{t}|x_{t+1}, x_0)$ for each step. \citet{song2020score} show that $J$ is equivalent to score-matching loss when formulating the forward and backward process as a discrete version of stochastic differential equations. 

\subsection{Integral Probability Metrics (IPM)}

Given $\mathcal{A}$ as a set of parameters s.t. for each $\alpha \in \mathcal{A}$, it defines a critic $f_{\alpha}: \mathbb{R}^n \rightarrow \mathbb{R}$.
Given a critic $f_\alpha$ and two distributions $p_{0}^{\theta}$ and $q_0$, we define
\begin{equation}
    g(p_0^\theta, f_\alpha, q_0) := \mathop{\mathbb{E}}_{x_0\sim p_{0}^{\theta}}[f_\alpha(x_0)]-\mathop{\mathbb{E}}_{x_0 \sim q_0}[f_\alpha(x_0)].
\end{equation}
Let
\begin{equation}
\Phi(p_{0}^{\theta},q_0) := \sup_{\alpha \in \mathcal{A}} g(p_0^\theta, f_\alpha, q_0).
\end{equation}

If $\mathcal{A}$ satisfies that  $\forall \alpha \in \mathcal{A}$, $\exists \alpha'  \in \mathcal{A}$, s.t. $f_{\alpha'} = -f_{\alpha}$, then $\Phi(p_{\theta},q)$ is a pseudo metric over the probability space of $\mathbb{R}^n$, making it so-called integral probability metrics (IPM). 

In this paper, we consider $\mathcal{A}$ that makes $\Phi(p_{0}^{\theta},q_0)$ an IPM. For example, when $\mathcal{A} = \{\alpha: ||f_{\alpha}||_{L}\leq 1\}$, $\Phi(p_{0}^{\theta},q_0)$ is the Wasserstein-1 distance; when $\mathcal{A} = \{\alpha: ||f_{\alpha}||_{\infty}\leq 1\}$, $\Phi(p_{0}^{\theta},q_0)$ is the total variation distance; it also includes maximum mean discrepancy (MMD) when $\mathcal{A}$ defines all functions in  Reproducing Kernel Hilbert Space (RKHS).







\section{Motivation}

\subsection{Issues with Existing DDPM Samplers}
\label{sec: motivation}
Here we review the existing issues with DDPM samplers 1) when $T$ is not large enough, and 2) when sub-sampling with the number of steps $T' \ll T$, which inspires us to design our fine-tuning algorithm.

\paragraph{Case 1. Issues caused by training DDPM with a small $T$ (Fig~\ref{fig:smallt}).} Given a score-matching loss $J$, the upper bound on Wasserstein-2 distance is given by ~\citet{kwon2022scorebased}:
\begin{equation}
W_2(p_0^\theta, q_0) \leq \mathcal{O}(\sqrt{J})+I(T)W_2(p_T, q_T),
\label{eq: upper bound}
\end{equation}
where $I(T)$ is non-exploding and $W_2(p_T, q_T)$ decays exponentially with $T$ when $T \rightarrow \infty$. From the inequality above, one sufficient condition for the score-matching loss $J$ to be viewed as optimizing the Wasserstein distance is when $T$ is large enough such that $I(T)W_2(p_T, q_T) \rightarrow 0$. 
%
Now we consider the case when $T$ is small and $p_T \not\approx q_T$.\footnote{Recall that during the diffusion process, we need small Gaussian noise for each step set the sampling chain to also be conditional Gaussian~\citep{ho2020denoising}. As a result, a small $T$ means $q_T$ is not close to pure Gaussian, and thus $p_T \not\approx q_T$.}. The upper bound in Eq.~(\ref{eq: upper bound}) can be high since $W_2(p_T,q_T)$ is not neglectable. As shown in Fig~\ref{fig:smallt}, pure imitation $p_t^{\theta}(x_t|x_{t+1}) \approx q(x_t|x_{t+1},x_0)$ would not lead the model exactly to $q_0$ when $p_T$ and $q_T$ are not close enough.
\paragraph{Case 2. Issues caused by a smaller number of sub-sampling steps ($T' \ll T$) (Fig~\ref{fig:fastsampling} in Appendix~\ref{app: subsampling}).} We consider DDPM sub-sampling and other fast sampling techniques, where $T$ is large enough s.t. $p_T \approx q_T$, but we try to sample with fewer sampling steps ($T'$). It is generally done by choosing $\tau$ to be an increasing sub-sequence of $T'$ steps in $[0, T]$ starting from $0$. Many works have been dedicated to finding a subsequence and variance schedule to make the sub-sampling steps match the full-step backward process as much as possible~\cite{kong2021fast,bao2021analytic,bao2022estimating}. However, this would inevitably cause downgraded sample quality if each step is Gaussian: as discussed in \citet{salimans2021progressive} and \citet{xiao2021tackling}, a multi-step Gaussian sampler cannot be distilled into a one-step Gaussian sampler without loss of fidelity. 


\subsection{Problem Formulation}
In both cases mentioned above,  there might exist paths other than imitating
the backward process that can reach the data distribution
with fewer Gaussian steps. Thus one may expect to
overcome these issues by minimizing the IPM.

Here we present the formulation of our problem setting. We assume that there is a target data distribution $q_0$. Given a set of critic parameters $\mathcal{A}$ s.t. $\Phi(p_{0}^{\theta},q_0) = \sup_{\alpha \in \mathcal{A}} g(p_0^\theta, f_\alpha, q_0)$ is an IPM, and given a DDPM sampler with $T$ steps parameterized by $\theta$, our goal is to solve: 
\begin{align}
    \min_{\theta} \Phi(p_{0}^{\theta},q_0).
\end{align}

\subsection{Pathwise Derivative Estimation for Shortcut Fine-Tuning: Properties and Potential Issues}
\label{sec: pathwise}
One straightforward approach is to optimize $\Phi(p_0^\theta,q_0)$ using pathwise derivative estimation~\citep{rezende2014stochastic} like GAN training, which we denote as \textbf{SFT} (shortcut fine-tuning). We can recursively define the stochastic mappings:
\begin{equation}
h_{\theta, T}(x_T) := x_T,  
\end{equation}
\begin{equation}
h_{\theta, t}(x_t) := \mu_{\theta}(h_{\theta, t+1}(x_{t+1})) +\epsilon_{t+1},
\end{equation}
\begin{equation}
x_0 = h_{\theta, 0}(x_T)
\end{equation}
where $x_T\sim \mathcal{N}(0,I), \epsilon_{t+1} \sim \mathcal{N}(0,\Sigma_{t+1}), t = 0,...,T-1$. 

Then we can write the objective function as:
\begin{equation}
    \Phi(p_0^\theta, q_0) =  \sup_{\alpha \in \mathcal{A}} \mathop{\mathbb{E}}_{x_T, \epsilon_{1:T}}[f_\alpha(h_{\theta,0}(x_T))]-\mathop{\mathbb{E}}_{x_0 \sim q_0}[f_\alpha(x_0)]
\end{equation}

Assume that $\exists \alpha \in \mathcal{A}$, s.t. $g(p_0^\theta, \alpha, q_0) = \Phi(p_{0}^{\theta},q_0)$. Let $\alpha^*(p_0^{\theta}, q_0) \in \{\alpha:g(p_0^\theta, \alpha, q_0) = \Phi(p_{0}^{\theta},q_0) \}$. When $f_{\alpha}$ is 1-Lipschitz, we can compute the gradient which is similar to WGAN~\citep{arjovsky2017wasserstein}:
\begin{equation}
\nabla_{\theta}\Phi(p_{0}^{\theta},q_0)=\underset{x_T, \epsilon_{1:T}}{\mathbb{E}}\left[\nabla_{\theta}f_{\alpha^*(p_0^{\theta}, q_0)}(h_{\theta,0}(x_T))\right].
\label{eq:wgan}
\end{equation}

\paragraph{Implicit requirements on the family of critics $\mathcal{A}$: gradient regularization.} In Eq.~(\ref{eq:wgan}), we can observe that the critic $f_{\alpha^*}$ needs to provide meaningful gradients (w.r.t. the input) for the generator. If the gradient of the critic happens to be 0 at some generated data points, even if the critic's value could still make sense, the critic would provide no signal for the generator on these points\footnote{For example, MMD with very narrow kernels can produce such critic functions, where each data point defines the center of the corresponding kernel which yields gradient 0.}. Thus GANs trained with IPMs generally need to choose $\mathcal{A}$ such that the gradient of the critic is regularized: For example, Lipschitz constraints like weight clipping~\citep{arjovsky2017wasserstein} and gradient penalty~\citep{gulrajani2017improved} for WGAN, and gradient regularizers for MMD GAN~\citep{arbel2018gradient}. 
\paragraph{Potential issues.}Besides the implicit requirements on the critic, there might also be issues when computing Eq.~(\ref{eq:wgan}) in practice. Differentiating a composite function with $T$ steps can cause problems similar to RNNs:
Gradient vanishing may result in long-distance dependency being lost;
Gradient explosion may occur;
Memory usage is high.


\section{Method: Shortcut Fine-Tuning with Policy Gradient (SFT-PG)}
We note that Eq.~(\ref{eq:wgan}) is not the only way to estimate the gradient w.r.t. IPM. In this section, we show that performing gradient descent of $\Phi(p_0^\theta, q_0)$ can be equivalent to policy gradient (Section~\ref{sec: pg}), provide analysis towards monotonic improvement (Section~\ref{sec: mono}) and then present the algorithm design (Section~\ref{sec: algo}).


\subsection{Policy Gradient Equivalence}

\label{sec: pg}

By modeling the conditional probability through the trajectory, we provide an alternative way for gradient estimation which is equivalent to policy gradient, without differentiating through the composite functions.

\begin{theorem} (\textbf{Policy gradient equivalence})

Assume that both
$p_{x_{0:T}}^\theta(x_{0:T})f_{\alpha^*(p_0^{\theta}, q_0)}(x_0)$ and $\nabla_{\theta}p_{x_{0:T}}^\theta(x_{0:T})f_{\alpha^*(p_0^{\theta}, q_0)}(x_0)$ are continuous functions w.r.t. $\theta$ and $x_{0:T}$.
Then
\begin{equation}
\begin{aligned}
&\nabla_{\theta}\Phi(p_{0}^{\theta},q_0)=\underset{p_{x_{0:T}}^\theta}{\mathbb{E}}\big[f_{\alpha^*(p_0^{\theta}, q_0)}(x_0) \sum_{t=0}^{T-1}\nabla_{\theta}\log p_t^{\theta}(x_t|x_{t+1}) \big].
\end{aligned}
\label{eq: pg}
\end{equation}
\end{theorem}
\begin{proof}
\begin{equation}
\begin{aligned}
&\mathrel{\phantom{=}}\nabla_{\theta}\Phi(p_{0}^{\theta},q_0)\\
&=  \nabla_{\theta}\int p_{0}^\theta(x_0)f_{\alpha^*(p_0^{\theta}, q_0)}(x_0) dx_0\\
&\mathrel{\phantom{=}}+ \nabla_\theta \alpha^*(p_0^{\theta}, q_0) \nabla_{\alpha^*(p_0^{\theta}, q_0)}\int p_{0}^\theta(x_0)f_{\alpha^*(p_0^{\theta}, q_0)}(x_0) dx_0,\\
\end{aligned}
\end{equation}
where $\nabla_{\alpha^*(p_0^{\theta}, q_0)}\int p_{0}^\theta(x_0)f_{\alpha^*(p_0^{\theta}, q_0)}(x_0) dx_0$ is 0 from the envelope theorem. Then we have
\begin{equation}
\begin{aligned}
&\mathrel{\phantom{=}}\nabla_{\theta}\int p_{0}^\theta(x_0)f_{\alpha^*(p_0^{\theta}, q_0)}(x_0) dx_0 \\
&=  \nabla_{\theta}\int \left(\int p_{x_{0:T}}^{\theta}(x_{0:T}) dx_{1:T}\right)f_{\alpha^*(p_0^{\theta}, q_0)}(x_0) dx_0, \\
&=\nabla_{\theta}\int  p_{x_{0:T}}^\theta(x_{0:T})f_{\alpha^*(p_0^{\theta}, q_0)}(x_0) dx_{0:T} \\
&= \int p_{x_{0:T}}^\theta(x_{0:T}) f_{\alpha^*(p_0^{\theta}, q_0)}(x_0) \nabla_{\theta}\log p^{\theta}_{x_{0:T}}(x_{0:T}) dx_{0:T} \\
&= \underset{p_{x_{0:T}}^\theta}{\mathbb{E}}\left[f_{\alpha^*(p_0^{\theta}, q_0)}(x_0) \sum_{t=0}^{T-1} \nabla_{\theta}\log p_t^{\theta}(x_t|x_{t+1})  \right],
\end{aligned}
\end{equation}
where the second last equality is from the continuous assumptions to exchange integral and derivative and the log derivative trick. The proof is then complete.
\qedhere
\end{proof}





\paragraph{MDP construction for policy gradient equivalence.} Here we explain why Eq.~(\ref{eq: pg}) could be viewed as policy gradient. We can construct an MDP with a finite horizon $T$: Treat $p_t^{\theta}(x_t|x_{t+1})$ as a policy, and assume that  transition is an identical mapping such that the action is to choose the next state. Consider reward as
$f_{\alpha^*(p_0^{\theta}, q_0)}(x_0)$ at the final step, and as $0$ at any other steps. Then Eq.~(\ref{eq: pg}) is equivalent to performing policy gradient~\citep{williams1992simple}.

\textbf{Comparing Eq.~(\ref{eq:wgan}) and  Eq.~(\ref{eq: pg}):} 

\begin{itemize}
    \item Eq.~(\ref{eq:wgan}) uses the gradient of the critic, while Eq.~(\ref{eq: pg}) only uses the value of the critic. This indicates that for policy gradient, weaker conditions are required for critics to provide meaningful guidance for the generator, which means more choices of $\mathcal{A}$ can be applied here.
    \item We compute the sum of gradients for each step in Eq.~(\ref{eq: pg}), which does not suffer from exploding or vanishing gradients. Also, we do not need to track gradients of the generated sequence during $T$ steps.

    \item However, stochastic policy gradient methods usually suffer from higher variance~\citep{mohamed2020monte}. Thanks to similar techniques in RL, we can reduce the variance via a baseline trick, which will be discussed in Section~\ref{sec: baseline}. 
\end{itemize}


In conclusion, Eq.~(\ref{eq: pg}) is comparable to Eq.~(\ref{eq:wgan}) in expectation, with potential benefits like numerical stability, memory efficiency, and a wider range of the critic family $\mathcal{A}$. It could suffer from higher variance but the baseline trick can help. We denote such kind of method as \textbf{SFT-PG} (shortcut fine-tuning with policy gradient).



\paragraph{Empirical comparison.} We conduct experiments on some toy datasets (Fig~\ref{fig: comparison}), where we show the performance of Eq.~(\ref{eq: pg}) with the baseline trick is at least comparable to Eq.~(\ref{eq:wgan}) at convergence when they use the same gradient penalty (GP) for critic regularization. We further observe SFT-PG with a newly proposed baseline regularization (B) enjoys a noticeably better  performance compared to SFT with GP. The regularization methods will be introduced in Section~\ref{sec: regularization}. Experimental details are in Section~\ref{sec: regularization exp}.

\subsection{Towards Monotonic Improvement}
\label{sec: mono}
\begin{figure}
\centering
\begin{subfigure}[b]{0.75\linewidth}
\centering
\includegraphics[width=0.6\linewidth]{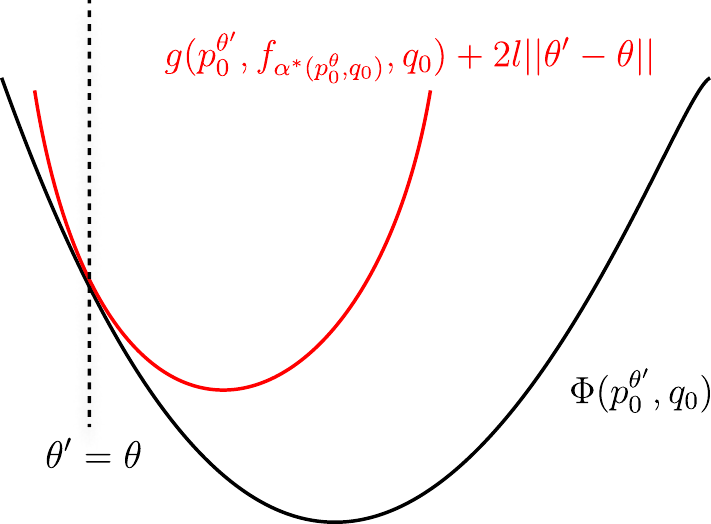}
\end{subfigure}
\caption{Illustration of the surrogate function given a fixed critic (red), and the actual objective $\Phi(p_{0}^{\theta'},q_{0})$ (dark). The horizontal axis represents the variable $\theta'$. Starting from $\theta$, a descent in the surrogate function is a sufficient condition for a descent in $\Phi(p_{0}^{\theta'},q_{0})$.}
\label{fig: surrogate}
\end{figure}

The gradient update discussed in Eq.~(\ref{eq: pg}) only supports one step of gradient update, given a fixed critic $f_{\alpha^*(p_0^\theta, q_0)}$ that is optimal to the current $\theta$. Questions remain: When is our update guaranteed to get improvement?  Can we do more than one update to get a potential descent? We answer the questions by providing a surrogate function of the IPM.


\begin{theorem} (\textbf{The surrogate function of IPM})

Assume that $g(p_0^{\theta}, f_\alpha, q_0)$ is Lipschitz w.r.t. $\theta$, given $q_0$ and $\alpha\in\mathcal{A}$. Given a fixed critic $f_{\alpha^*(p_0^\theta, q_0)}$, there exists $l\geq 0$ such that $\Phi(p_0^{\theta'}, q_0)$ is upper bounded by the surrogate function below:
\begin{equation}
    \Phi(p_0^{\theta'}, q_0) \leq g(p_0^{\theta'},f_{\alpha^*(p_0^\theta,q_0)}, q_0) + 2l||\theta'-\theta||.
\end{equation}
\label{mono}
\end{theorem}

Proof of Theorem~\ref{mono} can be found in Appendix~\ref{app: mono}. Here we provide an illustration of Theorem~\ref{mono} in Fig~\ref{fig: surrogate}. Given a critic that is optimal w.r.t. $\theta$, $\Phi(p_0^{\theta'}, q_0)$ is unknown if $\theta \neq \theta'$. But if we can get a descent of the surrogate function, we are also guaranteed to get a descent of $\Phi(p_0^{\theta'}, q_0)$, which facilitates more potential updates even if $\theta' \neq \theta$.
Moreover, using the Lagrange multiplier, we can convert minimizing the surrogate function to a constrained optimization problem to optimize $g(p_0^{\theta'},f_{\alpha^*(p_0^\theta, q_0)}, q_0)$ with the constraint that $||\theta'-\theta||\leq \delta$ for some $\delta >0$. Following this idea, one simple trick is to perform $n_{\text{generator}}$ steps of gradient updates with a small learning rate, and clip the gradient norm with threshold $\gamma$. We present the empirical effect of such simple modification in Section~\ref{sec: clipping}, Table~\ref{gradnorm}.

\paragraph{Discussion.}
One may notice that Theorem~\ref{mono} is similar in spirit to Theorem 1 in TRPO~\citep{schulman2015trust}, which provides a surrogate function for a fixed but unknown reward function. In our case, the reward function $f_{\alpha^*(p_0^\theta, q_0)}$ is known for the current $\theta$ but changing: It is dependent on the current $\theta$ so it remains unknown for $\theta'\neq\theta$. The proof techniques are also different, but they both estimate an unknown part of the objective function.



\subsection{Algorithm Design}
\label{sec: algo}

In the previous sections, we only consider the case where we have an optimal critic function given $\theta$. In the training, we adopt similar techniques in WGAN~\citep{arjovsky2017wasserstein} to perform alternative training of the critic and generator in order to approximate the optimal critic. Consider the objective function below: 
\begin{equation} 
    \min_{\theta} \max_{\alpha\in \mathcal{A}}  g(p_0^\theta, f_\alpha, q_0).
\end{equation}

Now we discuss techniques to reduce the variance of the gradient estimation and regularize the critic, and then give an overview of our algorithm.


\subsubsection{Baseline Function for Variance Reduction}
\label{sec: baseline}

Given a critic $\alpha$, we can adopt a technique widely used in policy gradient to reduce the variance of the gradient estimation in Eq.~(\ref{eq: pg}). Similar to \citet{schulman2015high}, we can subtract a baseline function $V^{\omega}_{t+1}(x_{t+1})$ from the cumulative reward $f_{\alpha} (x_0)$, without changing the expectation:
\begin{equation}
\begin{aligned}
&\mathrel{\phantom{=}}\nabla_{\theta} g(p_0^\theta, f_\alpha, q_0)\\
&=
\underset{p_{x_{0:T}}^\theta}{\mathbb{E}}\left[f_{\alpha}(x_0) \sum_{t=0}^{T-1} \nabla_{\theta}\log p_t^{\theta}(x_t|x_{t+1}) \right]\\
&=\underset{p_{x_{0:T}}^\theta}{\mathbb{E}}\left[\sum_{t=0}^{T-1}(f_{\alpha}(x_0) - V_{t+1}^\omega(x_{t+1}))\nabla_{\theta}\log p_t^{\theta}(x_t|x_{t+1})  \right],
\end{aligned}
\label{eq: baseline}
\end{equation}
where the optimal choice of $V_{t+1}^\omega(x_{t+1})$ to minimize the variance would be $V_{t+1}(x_{t+1},\alpha):=\underset{p_{x_{0:T}}^\theta}{\mathbb{E}}[f_{\alpha}(x_0)|x_{t+1}]$. Detailed derivation of Eq.~(\ref{eq: baseline}) can be found in Appendix~\ref{app: baseline}.
Thus, given a critic $\alpha$ and a generator $\theta$, we can train a value function $V_{t+1}^\omega$ by minimizing the objective below:
\begin{equation}
    R_{B}(\alpha, \omega, \theta)=\underset{p_{x_{0:T}}^\theta}{ \mathbb{E}}\left[\sum_{t=0}^{T-1}(f_\alpha(x_0) -V_{t+1}(x_{t+1},\alpha))^2\right].
\end{equation}

\subsubsection{Choices of $\mathcal{A}$: Regularizing the Critic}
\label{sec: regularization}
Here we discuss different choices of $\mathcal{A}$, which indicates different regularization methods for the critic. 

\paragraph{Lipschitz regularization.} If we choose $\mathcal{A}$ to include parameters of all 1-Lipschitz functions, we can adopt regularization as WGAN-GP~\citep{gulrajani2017improved}:
\begin{equation}
    R_{GP}(\alpha, \theta) =  \underset{\hat{x_0}}{\mathbb{E}}\left[(||\nabla_{x_0}f_{\alpha}(x_0)||-1)^2\right],
\end{equation}
where $\hat{x_0}$ is sampled uniformly on the line segment between $x_0' \sim p_0^\theta$ and $x_0'' \sim q_0$. $f_\alpha$ can be trained to maximize $g(p_0^\theta, f_\alpha, q_0) -\eta R_{GP}(\alpha, \omega, \theta)$, $\eta>0$ is the regularization coefficient.

\paragraph{Reusing baseline for critic regularization.} As discussed in Section~\ref{sec: pg}, since we only use the critic value during updates, now we can afford a potentially wider range of critic family $\mathcal{A}$. Some regularization on $f_\alpha$ is still needed; Otherwise its value can explode. Also, regularization is shown to be beneficial for local convergence~\citep{mescheder2018training}.  So we consider regularization that can be weaker than gradient constraints, such that the critic is more sensitive to the changes of the generator, which could be favorable when updating the critic for a fixed number of training steps.


We found an interesting fact that the loss $R_{B}(\alpha,\omega,\theta)$ can be \emph{reused} to regularize the value of $f_\alpha$ instead of the gradient, which implicitly defines a set $\mathcal{A}$ that shows empirical benefits in practice. 

Define
\begin{equation}
L(\alpha, \omega, \theta) := g(p_0^\theta, f_\alpha, q_0) -\lambda R_{B}(\alpha, \omega, \theta).
\label{eq: b and c}
\end{equation}
Given $\theta$, our critic $\alpha$ and baseline $\omega$ can be trained together to maximize $L(\alpha, \omega, \theta)$. 

We provide an explanation of such kind of implicit regularization. During the update, we can view $V^\omega_{t+1}$ as an approximation of the expected value of $f_{\alpha}$ from the previous step. The regularization provides a trade-off between maximizing $g(p_0^\theta, f_\alpha, q_0)$ and minimizing changes in the expected value of $f_\alpha$, preventing drastic changes in the critic and stabilizing the training. Intuitively, it helps local convergence when both the critic and generator are already near-optimal: there is an extra cost for the critic value to diverge away from the optimal value. As a byproduct, it also makes the baseline function easier to fit since the regularization loss is reused.

\paragraph{Empirical comparison: baseline regularization and gradient penalty.} We present a comparison of gradient penalty (GP) and baseline regularization (B) for policy gradient training (SFT-PG) in Section~\ref{sec: regularization exp}, Fig~\ref{fig: comparison} on toy datasets, which shows in policy gradient training, the baseline function performs comparably well or even better than gradient penalty.

    

\subsubsection{Putting Together: Algorithm Overview}

Now we are ready to present our algorithm. 
Our critic $\alpha$ and baseline $\omega$ are trained to maximize $L(\alpha, \omega, \theta) = g(p_0^\theta, f_\alpha, q_0) -\lambda R_{B}(\alpha, \omega, \theta)$, and the generator is trained to minimize $g(p_0^\theta, f_\alpha, q_0)$ via Eq.~(\ref{eq: baseline}). To save memory usage, we use a buffer $\mathcal{B}$ that contains $\{x_{t+1}, x_t, x_0, t\}$ generated from the current generator without tracking the gradient, and randomly sample a batch from the buffer to compute Eq.~(\ref{eq: baseline}) and then perform backpropagation. The maximization and minimization steps are performed alternatively. See details in Alg~\ref{alg}.

\begin{algorithm}[]
  \caption{Shortcut Fine-Tuning with Policy Gradient and Baseline Regularization: SFT-PG (B)}

\textbf{Input}: $n_{\text{critic}}$, $n_{\text{generator}}$, batch size $m$, critic parameters $\alpha$, baseline function parameter $\omega$ , pretrained generator $\theta$, regularization hyperparameter $\lambda$

\begin{algorithmic}

\WHILE{$\theta$ not converged}
\STATE Initialize trajectory buffer $\mathcal{B}$ as $\emptyset$
 \FOR{$i$ = 0,...,$n_{\text{critic}}$}
\STATE Obtain $m$ i.i.d. samples from $p^{\theta}_{x_{0:T}}$ 
\STATE Add all $\{x_{t+1}, x_t, x_0, t\}$ to $\mathcal{B}$, $t=0,...,T-1$
\STATE Obtain $m$ i.i.d. samples from $q_0$ 
\STATE Update $\alpha$ and $\omega$ via maximizing Eq.~(\ref{eq: b and c})
\ENDFOR

\FOR{$j$ = 0,...,$n_{\text{generator}}$}
\STATE Obtain $m$ samples of $\{x_{t+1}, x_t, x_0, t\}$ from $\mathcal{B}$
\STATE Update $\theta$ via policy gradient according to Eq.~(\ref{eq: baseline}) 
\ENDFOR
\ENDWHILE
\end{algorithmic}
\label{alg}
\end{algorithm}
\vspace{-0.1cm}



\section{Related Works}
\label{sec: related}

\paragraph{GAN and RL.}
There are works using ideas from RL to train GANs~\citep{yu2017seqgan,wang2017irgan, sarmad2019rl, bai2019model}. The most relevant work is SeqGAN~\citep{yu2017seqgan}, which uses policy gradient to train the generator network. There are several main differences between their settings and ours. First, different GAN objectives are used: SeqGAN uses the JS divergence while we use IPM. In SeqGAN, the next token is dependent on tokens generated from all previous steps, while in diffusion models the next image is only dependent on the model output from one previous step; Also, the critic takes the whole generated sequence as input in SeqGAN, while we only care about the final output. Besides, in our work, rewards are mathematically derived from performing gradient descent w.r.t. IPM, while in SeqGAN, rewards are designed manually. In conclusion, different from SeqGAN, we propose a new policy gradient algorithm to optimize the IPM objective, with a novel analysis of monotonic improvement conditions and a new regularization method for the critic.

\paragraph{Diffusion and GAN.} There are other works combining diffusion and GAN training:  \citet{xiao2021tackling} consider multi-modal noise distributions generated by GAN to enable fast sampling; \citet{zheng2022truncated} considers a truncated forward process by replacing the last steps in the forward process with an autoencoder to generate noise, and start with the learned autoencoder as the first step of denoising and then continue to generate data from the diffusion model; Diffusion GAN~\citep{wang2022diffusion} perturbs the data with an adjustable number of steps, and minimizes JS divergence for all intermediate steps by training a multi-step generator with a time-dependent discriminator. To our best knowledge, there is no existing work using GAN-style training to fine-tune a pretrained DDPM sampler.

\paragraph{Fast samplers of DDIM and more.} There is another line of work on fast sampling of DDIM~\citep{song2020denoising}, for example, knowledge distillation~\citep{luhman2021knowledge,salimans2021progressive} and solving ordinary
differential equations (ODEs) with fewer steps~\citep{liu2022pseudo,lu2022dpm}. Samples generated by DDIM are generally less diverse than DDPM~\citep{song2020denoising}. 
Also, fast sampling is generally easier for DDIM samplers (with deterministic Markov chains) than DDPM samplers, since it is possible to combine multiple deterministic steps into one step without loss of fidelity, but not for combining multiple Gaussian steps as one~\citep{salimans2021progressive}. Fine-tuning DDIM samplers with deterministic policy gradient for fast sampling also seems possible, but deterministic policies may suffer from suboptimality, especially in high-dimensional action space~\citep{silver2014deterministic}, though it might require fewer samples. Also, it becomes less necessary since distillation is already possible for DDIM. 

Moreover, there is also some recent work that uses sample quality metrics to enable fast sampling. Instead of fine-tuning pretrained models, \citet{watson2021learning} propose to optimize the hyperparameters of the sampling schedule for a family of non-Markovian samplers by differentiating through KID~\citep{binkowski2018demystifying}, which is calculated by pretrained inception features. It is followed by a contemporary work that fine-tunes pretrained DDIM models using MMD calculated by  pretrained features~\citep{aiello2023fast}, which is similar to the method discussed in Section~\ref{sec: pathwise} but with a fixed critic and a deterministic sampling chain. Generally speaking, adversarially trained critics can provide stronger signals than fixed ones and are more helpful for training~\citep{li2017mmd}. As a result, besides the potential issues discussed in Section~\ref{sec: pathwise}, such training may also suffer from sub-optimal results when $p^\theta_0$ is not close enough to $q_0$ at initialization, and is highly dependent on the choice of the pretrained feature.



\section{Experiments}
\label{sec: full exp}

\begin{figure*}[ht]
\centering
\scalebox{1.0}{
\begin{subfigure}[b]{\linewidth}
    \begin{subfigure}[b]{0.24\linewidth}
     \includegraphics[width=\linewidth]{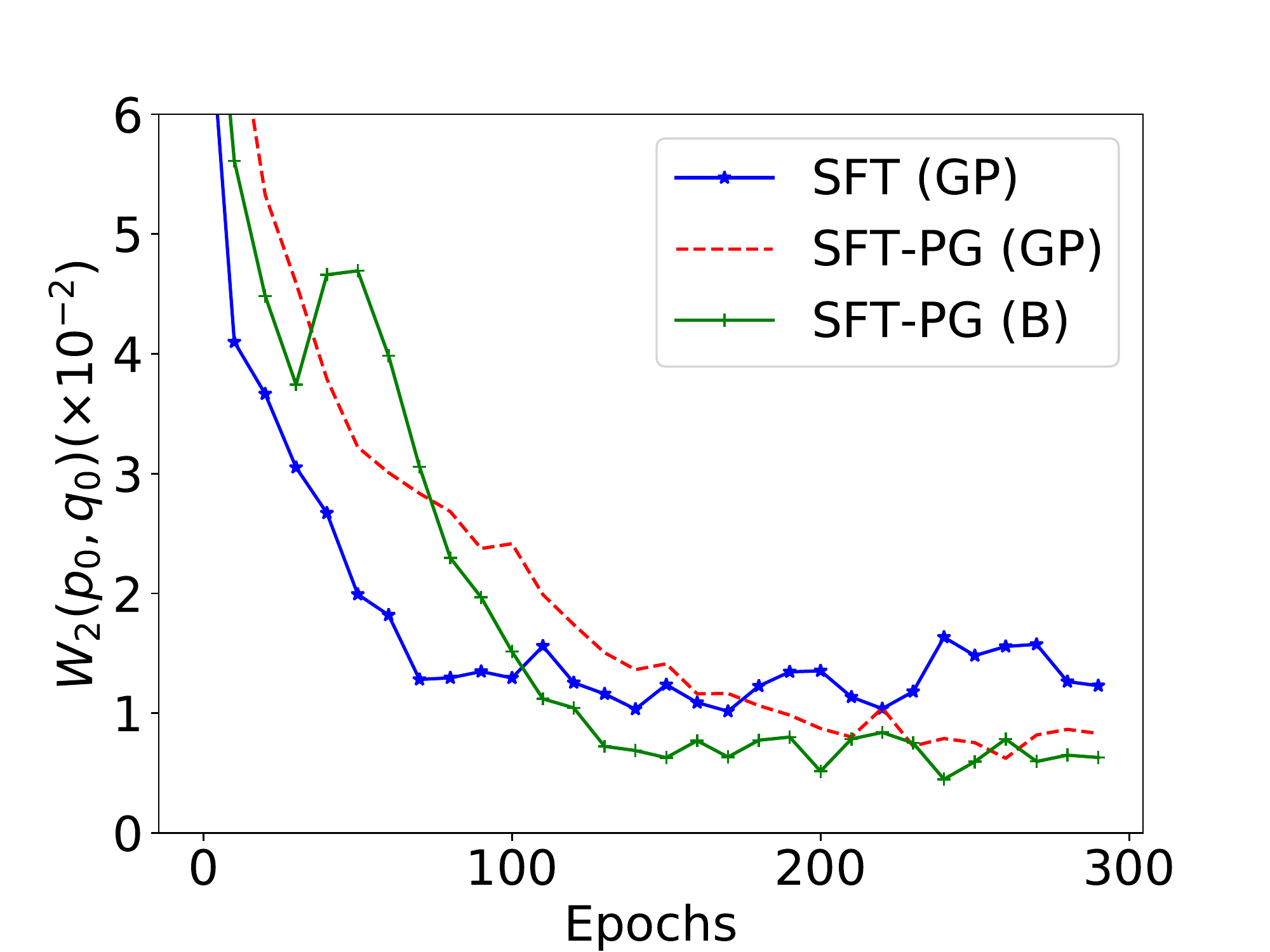}
\caption{Training curves of swiss roll}
\label{fig:swiss}
     \end{subfigure}
     \hfill
         \begin{subfigure}[b]{0.24\linewidth}
     \includegraphics[width=\linewidth]{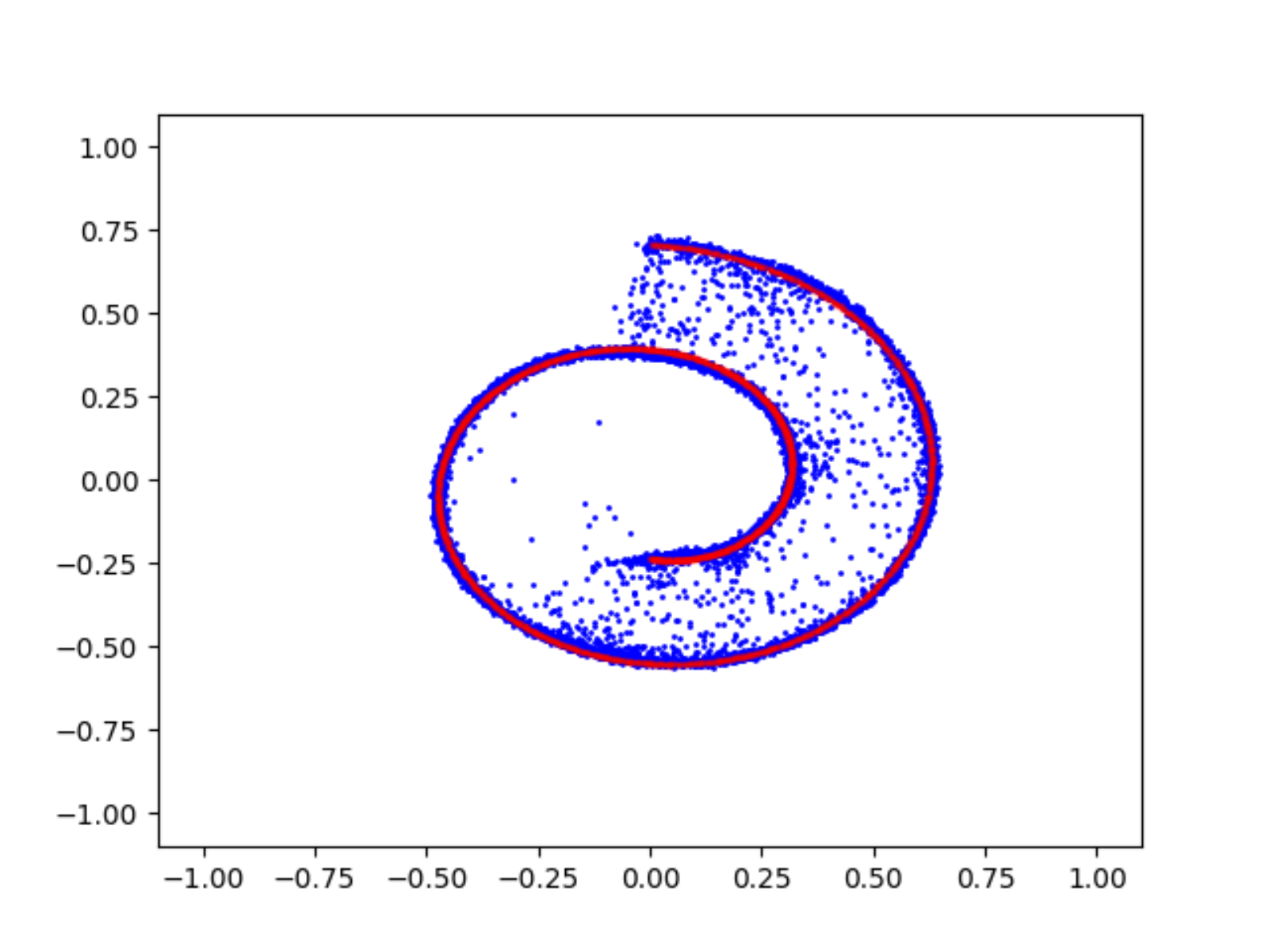}
\caption{Roll, SFT (GP)}
\label{fig:swiss-wgan}
     \end{subfigure}
    \hfill
    \begin{subfigure}[b]{0.24\linewidth}
\includegraphics[width=\linewidth]{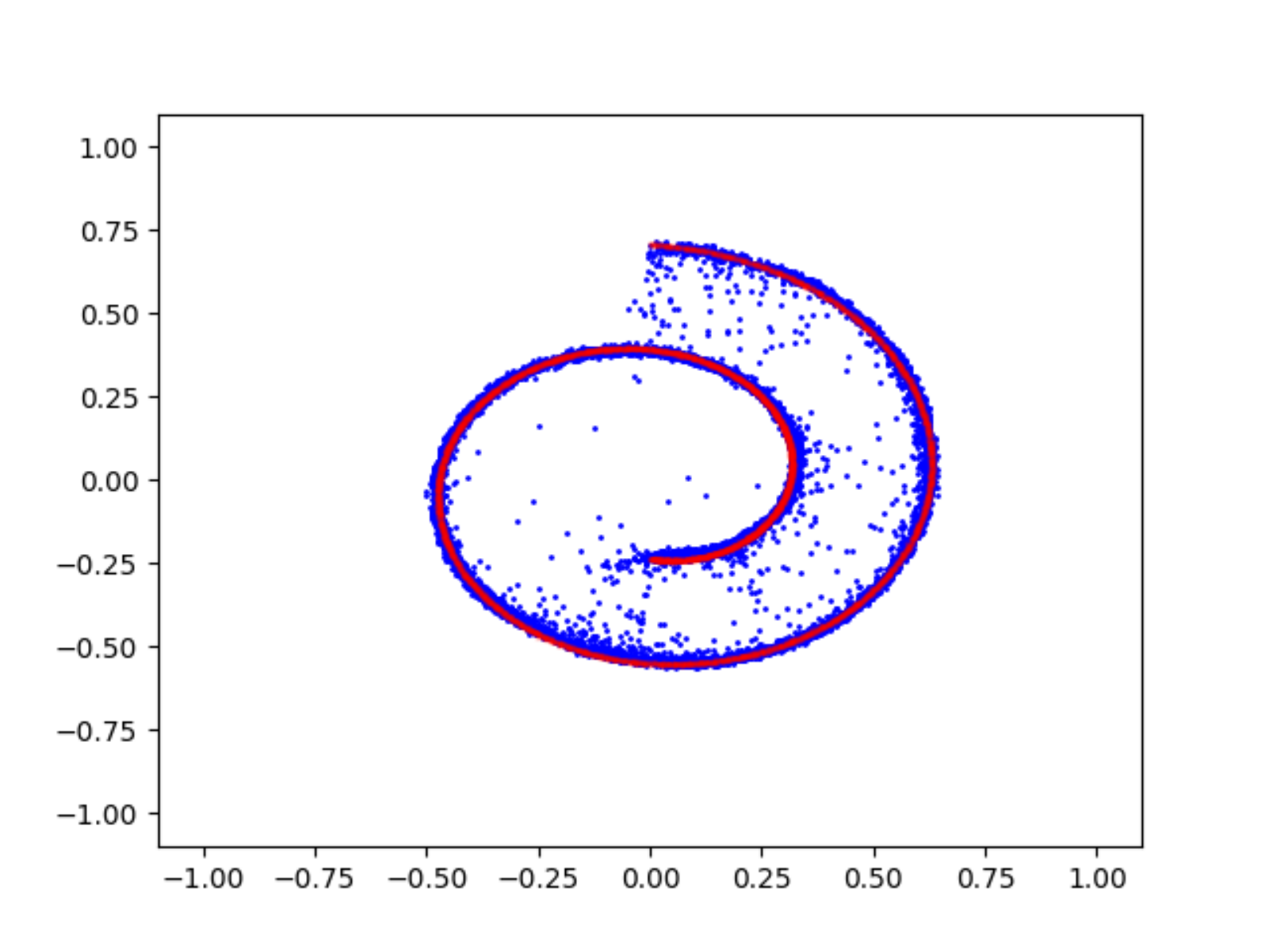}
\caption{Roll, SFT-PG (GP)}
\label{fig:swiss-gp}
\end{subfigure}
\hfill
    \begin{subfigure}[b]{0.24\linewidth}
\includegraphics[width=\linewidth]{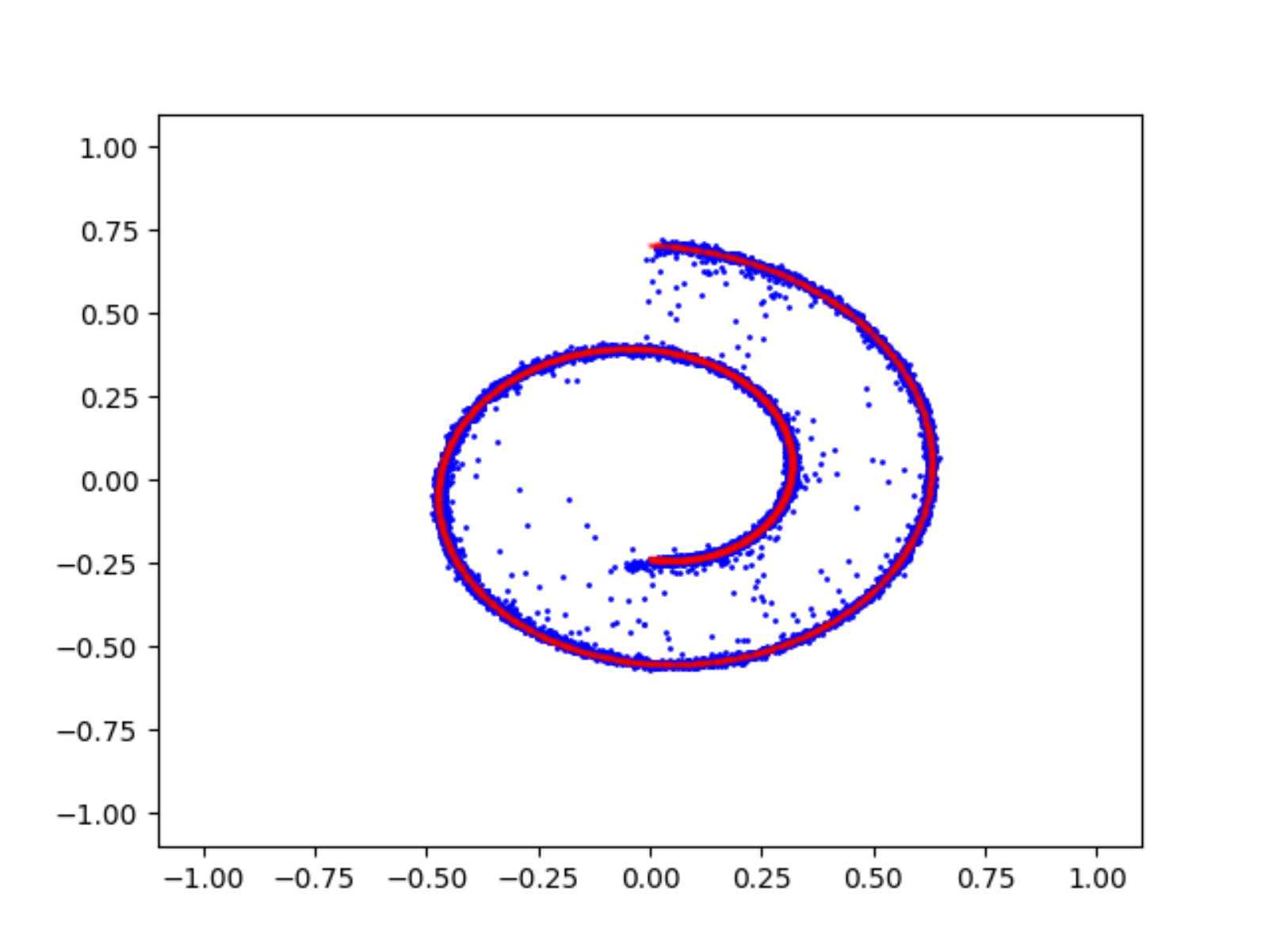}
\caption{Roll, SFT-PG (B)}
\label{fig:swiss-ft}
     \end{subfigure}
\vfill
\begin{subfigure}[b]{0.24\linewidth}
\includegraphics[width=\linewidth]{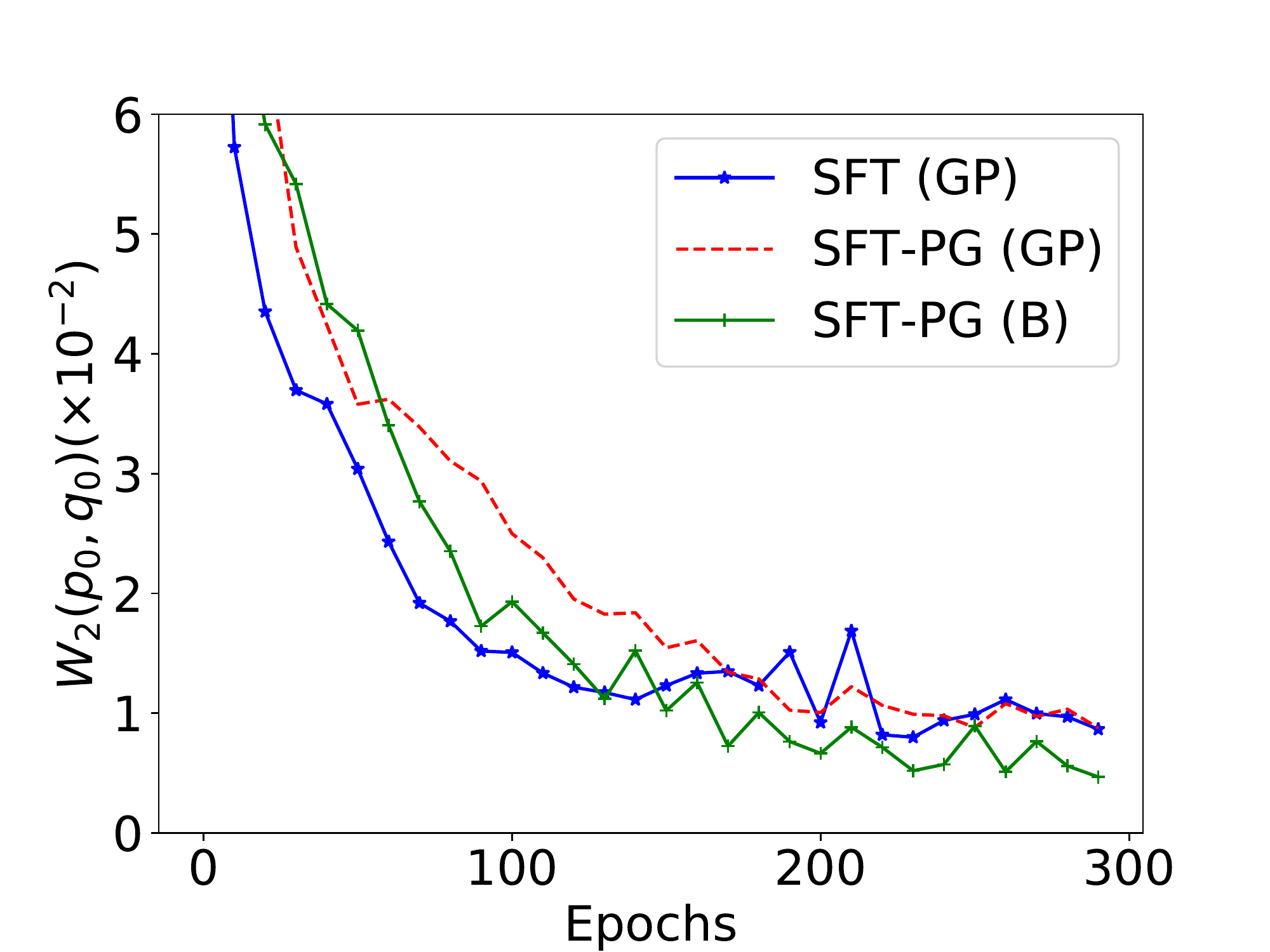}
\caption{Training curves of moons}
\label{fig:moon}
\end{subfigure}
\hfill
\begin{subfigure}[b]{0.24\linewidth}
\includegraphics[width=\linewidth]{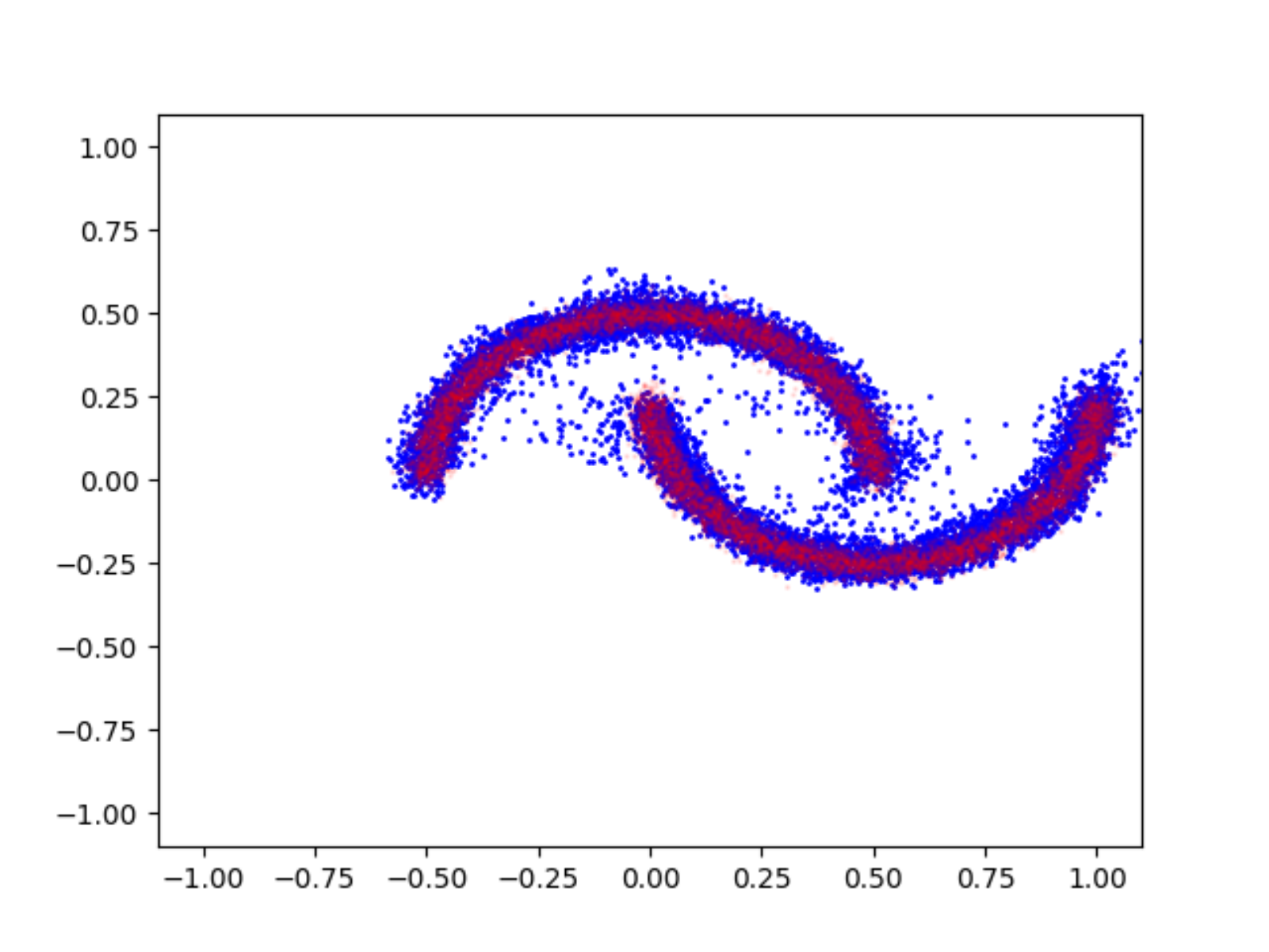}
\caption{Moons, SFT (GP)}
\label{fig:moon-wgan}
\end{subfigure}
\hfill
\begin{subfigure}[b]{0.24\linewidth}
\includegraphics[width=\linewidth]{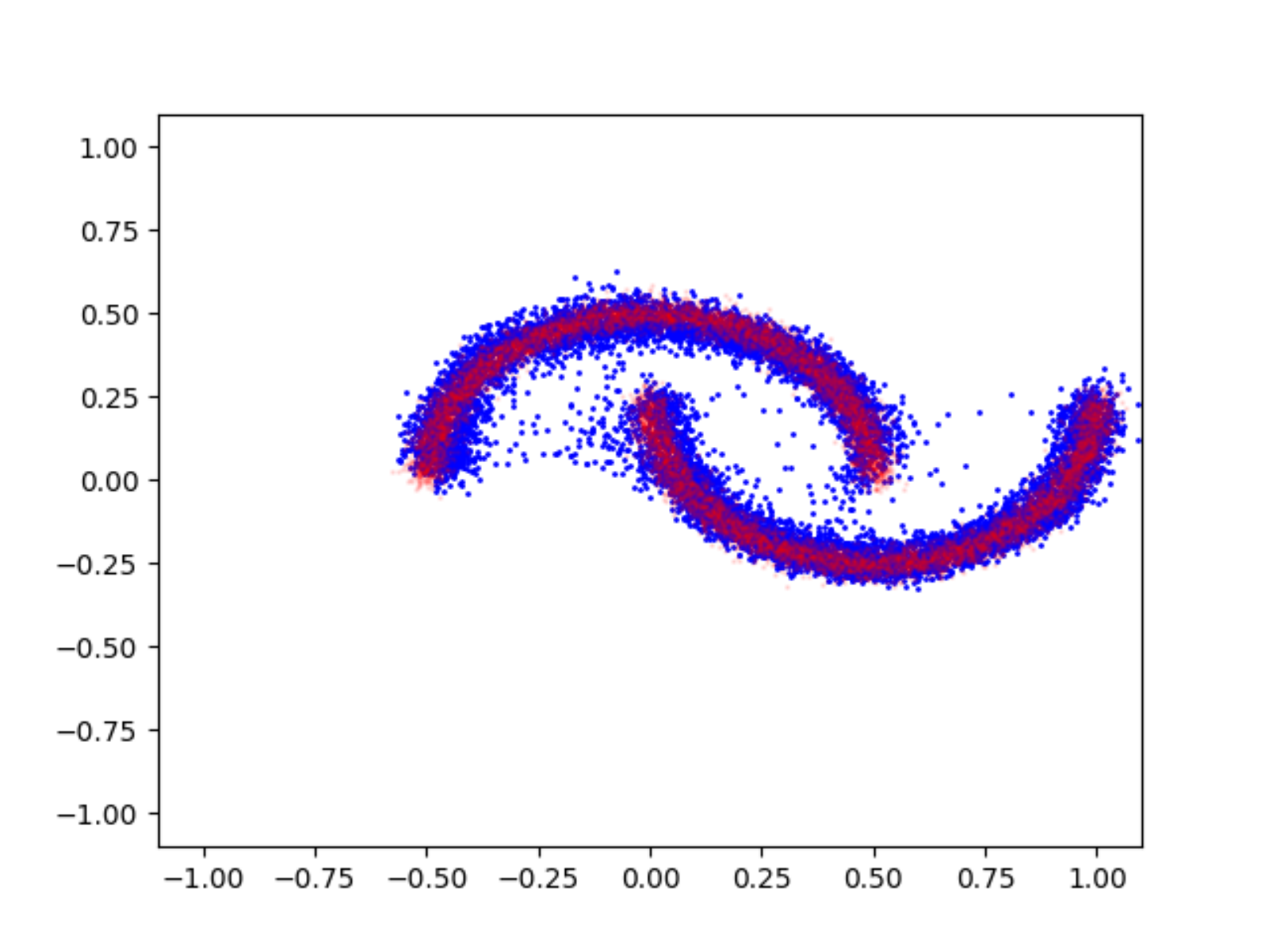}
\caption{Moons, SFT-PG (GP)}
\label{fig:moon-gp}
\end{subfigure}
\hfill
\begin{subfigure}[b]{0.24\linewidth}
\includegraphics[width=\linewidth]{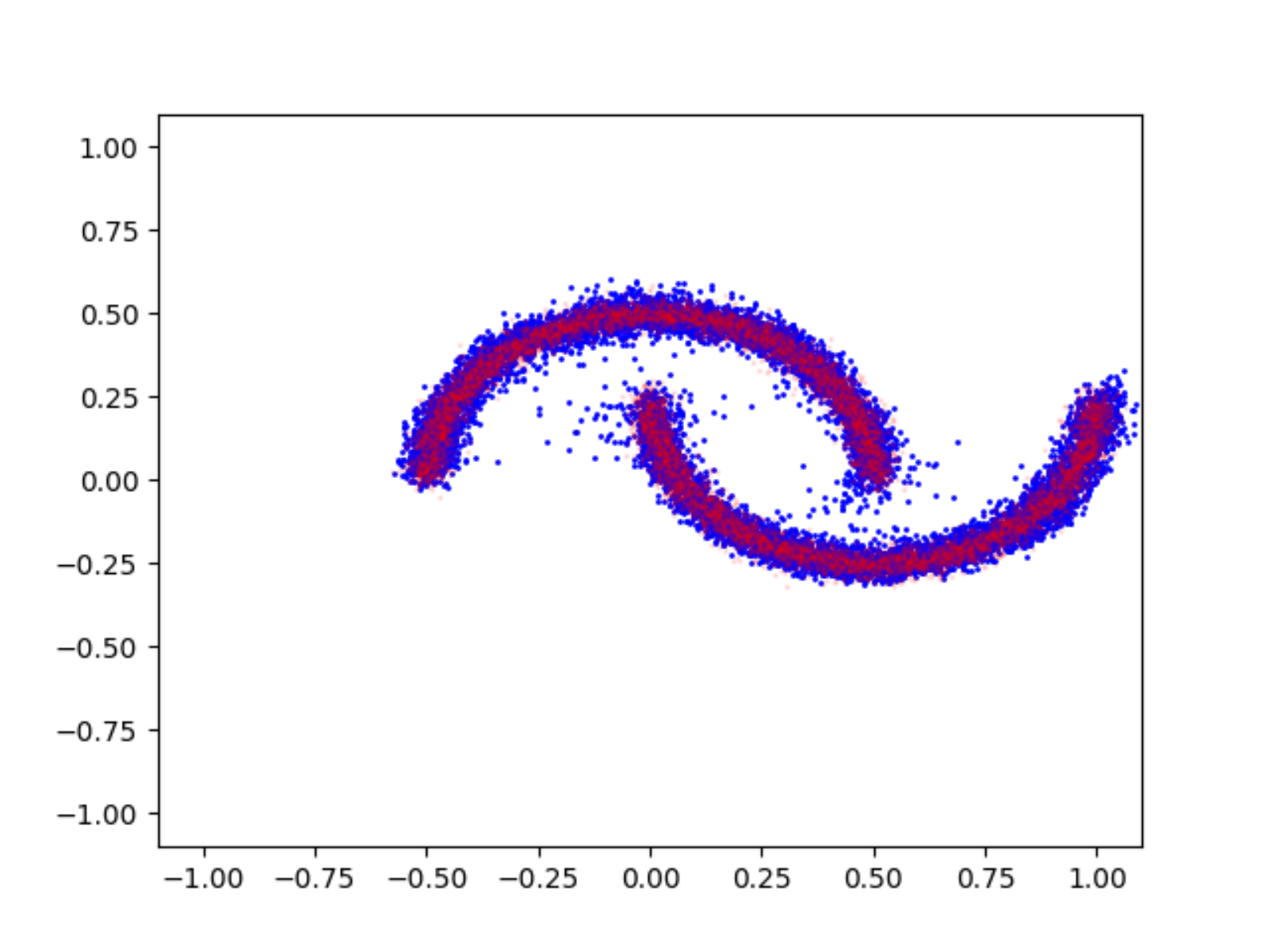}
\caption{Moons, SFT-PG (B)}
\label{fig:moon-ft}
\end{subfigure}
\end{subfigure}
}
\caption{Training curves~(\ref{fig:swiss}, \ref{fig:moon}) and 10K randomly generated samples from SFT (GP)~(\ref{fig:swiss-wgan}, \ref{fig:moon-wgan}),  SFT-PG (GP)~(\ref{fig:swiss-gp}, \ref{fig:moon-gp}), and SFT-PG  (B)~(\ref{fig:swiss-ft}, \ref{fig:moon-ft}) at convergence. In the visualizations, red dots indicate the ground truth data, and blue dots indicate generated data. We can observe that SFT-PG (B) produces noticeably better distributions, which is the result of utilizing a wider range of critics.}
\label{fig: comparison}
\end{figure*}

 

In this section, we aim to answer the following questions:
\begin{itemize}
    \item (Section~\ref{sec: poc}) Does the proposed algorithm SFT-PG (B) work in practice? 
    \item (Section~\ref{sec: regularization exp}) How does  SFT-PG (Eq.~(\ref{eq: pg})) work compared to SFT (Eq.~(\ref{eq:wgan})) with the same regularization (GP), and how does baseline regularization (B) compare to gradient penalty (GP) in SFT-PG?
    \item (Section~\ref{sec: clipping}) Do more generator steps with gradient clipping work as discussed in Section~\ref{sec: mono}?
    \item (Section~\ref{sec: benchmark}) Does the proposed fine-tuning SFT-PG (B) improve existing fast samplers of DDPM on benchmark datasets?
\end{itemize}

Code is available at \url{https://github.com/UW-Madison-Lee-Lab/SFT-PG}.
\subsection{Setup}
Here we provide the setup of our training algorithm on different datasets.  Model architectures and training details can be found in Appendix~\ref{app: exp}. 
\paragraph{Toy datasets.}
The toy datasets we use are swiss roll and two moons~\citep{scikit-learn}. We use $\lambda=0.1$, $n_{\text{critic}} = 5, n_{\text{generator}} = 1$ with no gradient clipping. For evaluation, we use the Wasserstein-2 distance on 10K samples from $p_0$ and $q_0$ respectively, calculated by POT~\citep{flamary2021pot}. 
\paragraph{Image datasets.}
We use MNIST~\citep{lecun1998gradient}, CIFAR-10~\citep{krizhevsky2009learning} and CelebA~\citep{liu2015deep}. For hyperparameters, we choose $\lambda=1.0$, $n_{\text{critic}} = 5, n_{\text{generator}} = 10$,  $\gamma=0.1$, except when testing different choices of $n_{\text{generator}}$ and $\gamma$ in MNIST, where we use $n_{\text{generator}} = 5$ and varying $\gamma$. For evaluation, we use FID~\citep{heusel2017gans} measured by 50K samples generated from $p_0^\theta$ and $q_0$ respectively. 









\begin{table}[h]
\centering
\vskip 0.08in
\scalebox{0.9}{
\begin{tabular}{lc}
\toprule
\textbf{Method}  &\textbf{$\mathbf{W_2(p_0^\theta,q_0)}$ ($\times 10^{-2}$) ($\downarrow$)}   \\
 \hline
$T= 10$, DDPM & 8.29\\
$T= 100$, DDPM  & 2.36\\
$T= 1000$, DDPM & 1.78 \\
$T=10$, SFT-PG (B) &\textbf{0.64} \\
 \bottomrule
\end{tabular}
}
\vspace{0.1in}
\caption{Comparison of DDPM models and our fine-tuned model on the swiss roll dataset.}
\label{tab: swiss}
\end{table}
\subsection{Proof-of-concept Results}

In this section, we fine-tune pretrained DDPMs with $T=10$, and present the effect of the proposed algorithm SFT-PG with baseline regularization on toy datasets. We present the results of different gradient estimations discussed in Section~\ref{sec: pg}, different critic regularization methods discussed in Section~\ref{sec: regularization}, and the training technique with more generator steps discussed Section~\ref{sec: mono}.

\subsubsection{Improvement from Fine-Tuning}
\label{sec: poc}
On the swiss roll dataset, we first train a DDPM with $T=10$ till convergence, and then use it as initialization of our fine-tuning. 
As in Table~\ref{tab: swiss}, our fine-tuned sampler with 10 steps can get better Wasserstein distance not only compared to the DDPM with $T=10$, but can even outperform DDPM with $T=1000$, which is reasonable since we directly optimize the IPM objective. \footnote{Besides, our algorithm also works when training from scratch with a final performance comparable to fine-tuning, but it will take longer time to train.}
The training curve and the data visualization can be found in Fig~\ref{fig:swiss} and Fig~\ref{fig:swiss-ft}.

%

\subsubsection{Effect of Different Gradient Estimations and Regularizations}
\label{sec: regularization exp}

On the toy datasets, we compare gradient estimation SFT-PG and SFT, both with gradient penalty (GP). \footnote{For gradient penalty coefficient, we tested different choices in $[0.001,10]$ and pick the best choice $0.001$. We also tried spectral normalization for Lipschitz constraints, but we found that its performance is worse than gradient penalty on these datasets.} We also compare them to our proposed algorithm SFT-PG (B).  All methods are initialized with pretrained DDPM, $T=10$, then trained till convergence. As shown in Fig~\ref{fig: comparison}, we can observe that all methods converge and the training curves are almost comparable, while SFT-PG (B) enjoys a slightly better final performance.







\begin{figure*}[ht]
\centering
\begin{subfigure}[b]{0.99\linewidth}
\begin{subfigure}[b]{0.24\linewidth}
    \centering
\includegraphics[width=\linewidth]{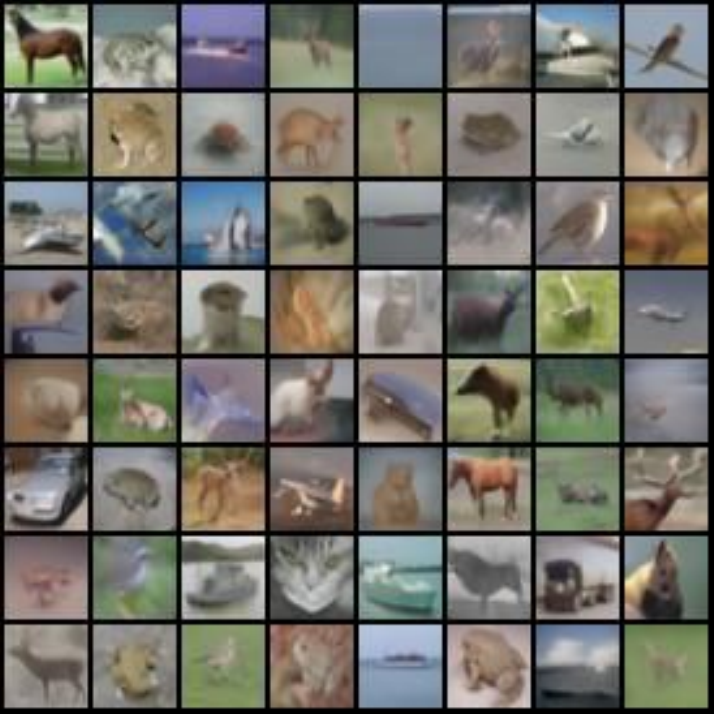}
\caption{CIFAR10, Initialization}
\end{subfigure}
\hfill
\begin{subfigure}[b]{0.24\linewidth}
    \centering
\includegraphics[width=\linewidth]{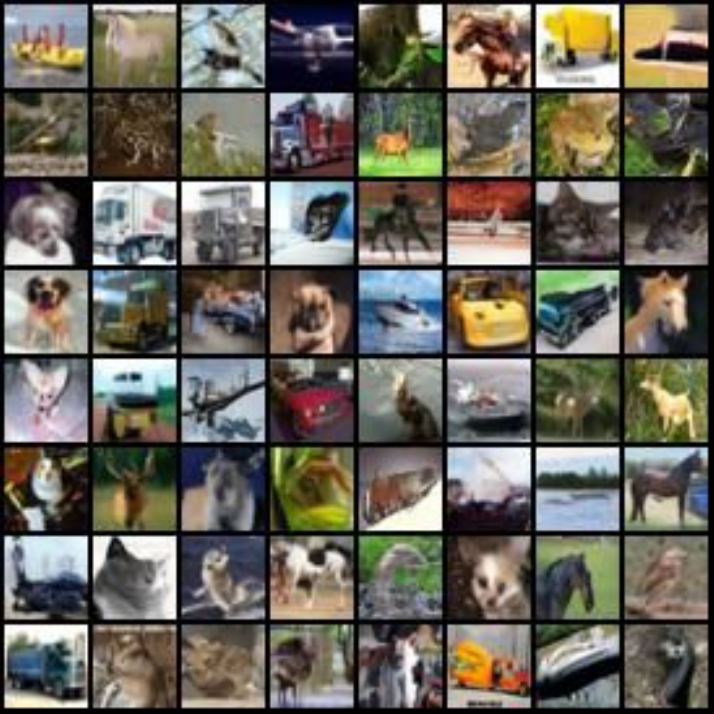}

\caption{CIFAR10, SFT-PG (B)}
\end{subfigure}
\hfill
\begin{subfigure}[b]{0.24\linewidth}
    \centering
\includegraphics[width=\linewidth]{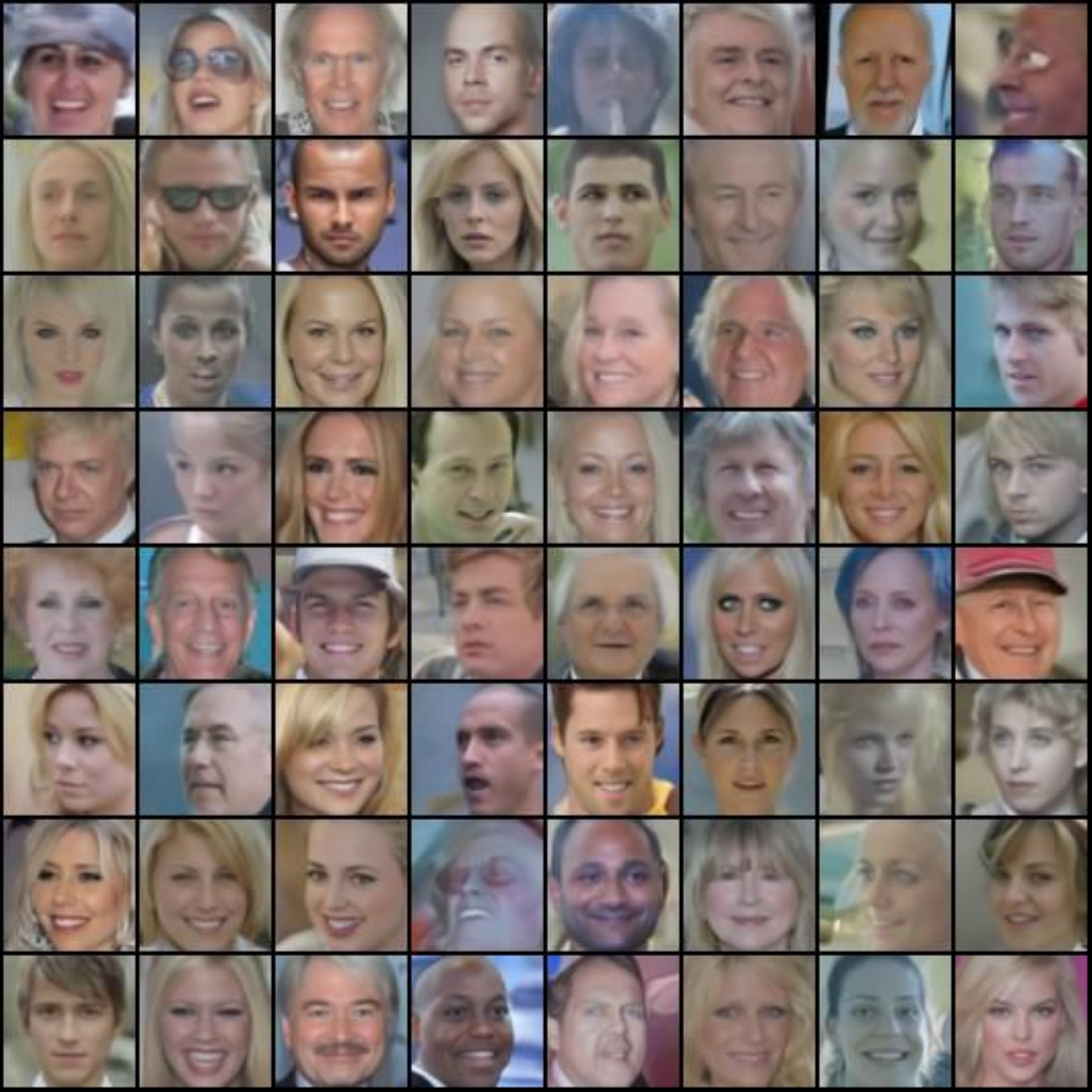}
\caption{CelebA, Initialization}
\end{subfigure}
\hfill
\begin{subfigure}[b]{0.24\linewidth}%
\centering
\includegraphics[width=\linewidth]{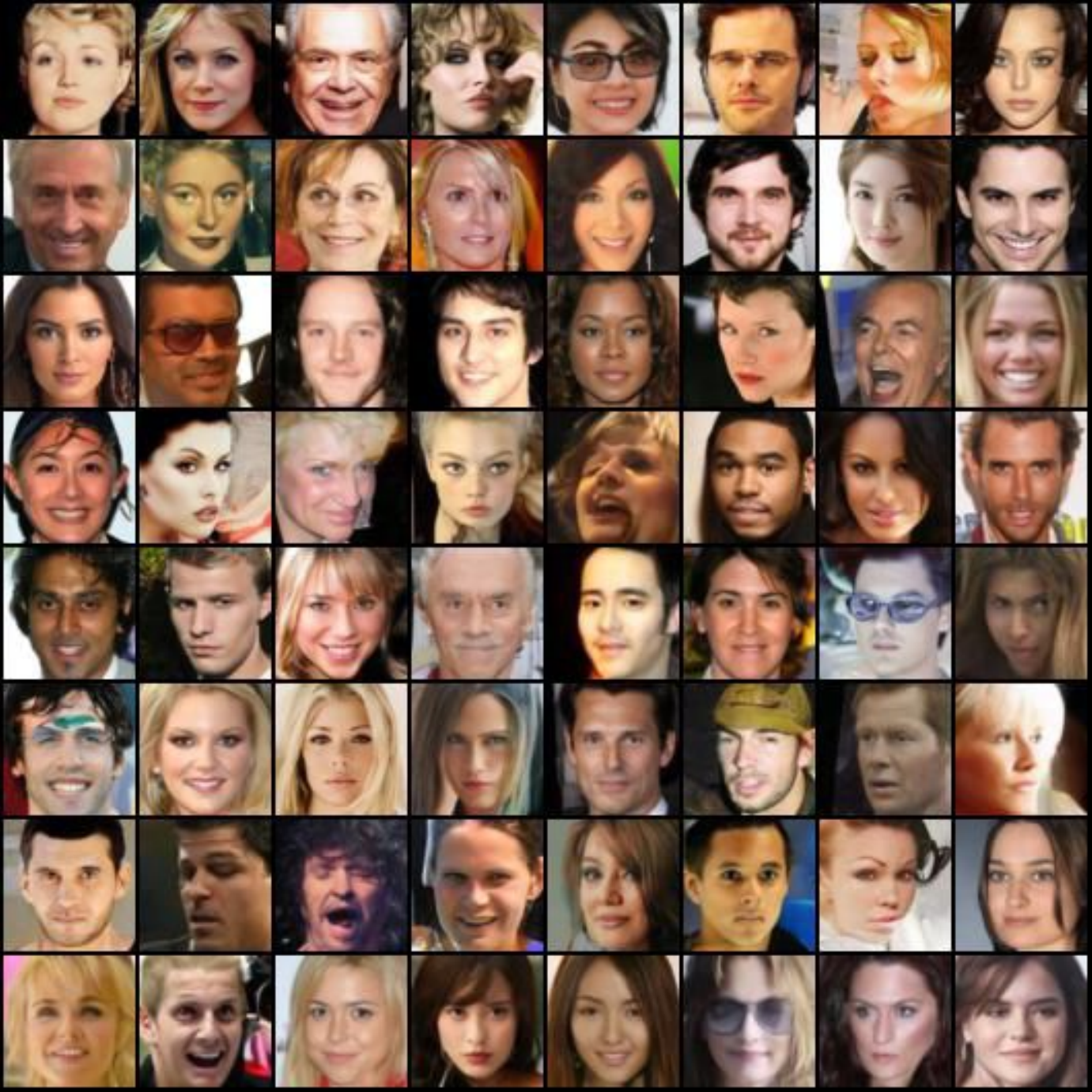}
\caption{CelebA, SFT-PG (B)}
\end{subfigure}
\end{subfigure}
     \caption{Randomly generated images before and after fine-tuning, on CIFAR10 $(32\times 32)$ and CelebA $(64\times 64)$, $T'=10$. The initialization is from pretrained models with $T=1000$ and sub-sampling schedules with $T'=10$ calculated from FastDPM~\citep{kong2021fast}.}
\label{fig: benchmark}
\end{figure*}
\subsubsection{Effect of Gradient Clipping with More Generator Steps}
\label{sec: clipping}

In Section~\ref{sec: mono}, we discussed that performing more generator steps with the same fixed critic and clipping the gradient norm can improve the training of our algorithm. Here we present the effect of $n_\text{generator} = 1$ or $5$ with different gradient clipping thresholds $\gamma$ on MNIST, initialized with a pretrained DDPM with $T=10$, FID=7.34. From Table~\ref{gradnorm}, we find that a small $\gamma$ with more steps can improve the final performance, but could hurt the performance if too small. Randomly generated samples from the model with the best FID are in Fig~\ref{mnist}.
We also conducted similar experiments on the toy datasets, but we find no significant difference on the final results, which is expected since the task is too simple.

\begin{minipage}{\linewidth}
\begin{minipage}[b]{0.49\linewidth}
\centering

\scalebox{0.8}{
\begin{tabular}[b]{ll}
\toprule
\textbf{Method}  &\textbf{FID ($\downarrow$)}   \\
 \hline
1 step & 1.35\\
5 steps, $\gamma = 10$ & 0.83\\
5 steps, $\gamma = 1.0$ & \textbf{0.82} \\
5 steps, $\gamma = 0.1$ &0.89 \\
5 steps, $\gamma = 0.001$ &1.46 \\
 \bottomrule
\end{tabular}
}
\captionof{table}{Effect of $n_{\text{generator}}$ and $\gamma$.}
\label{gradnorm}
\end{minipage}
\hfill
\begin{minipage}[b]{0.49\linewidth}
\centering
\scalebox{1.12}{\includegraphics[width=0.49\linewidth]{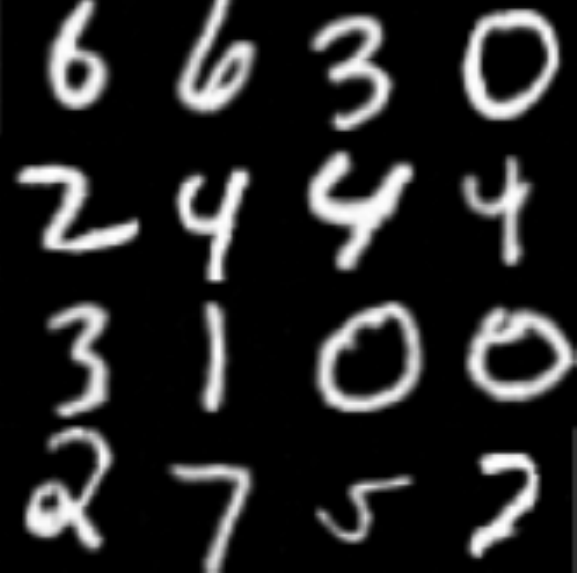}}
\captionof{figure}{Generated samples.}
\label{mnist}
\end{minipage}
\end{minipage}






\subsection{Benchmark Results}
\label{sec: benchmark}
To compare with existing fast samplers of DDPM, we take pretrained DDPMs with $T=1000$ and fine-tune them with sampling steps $T'=10$ on image benchmark datasets CIFAR-10 and CelebA.

Our baselines include various fast DDPM samplers with Gaussian noises:
naive DDPM sub-sampling, FastDPM~\citep{kong2021fast}, and recently advanced samplers like Analytic DPM~\citep{bao2021analytic} and SN-DPM~\citep{bao2022estimating}.  For fine-tuning, we use the fixed variance and sub-sampling schedules computed by FastDPM with $T'=10$ and only train the mean prediction model. From Table~\ref{tab: benchmark}, we can observe that the  performance of fine-tuning with $T'=10$ 
is comparable to the pretrained model with $T=1000$, outperforming the existing fast DDPM samplers. Randomly generated images before and after fine-tuning are in Fig~\ref{fig: benchmark}. 

We also present a comparison with DDIM sampling methods on CIFAR 10 benchmark in Appendix~\ref{app: ddim}, where our method is comparable to progressive distillation with $T' = 8$. 

\vspace{-0.3cm}



\subsection{Discussions and Limitations}
In our experiments, we only train $\mu^\theta_{t}$ given a pretrained DDPM. 
It is also possible to learn the variance via fine-tuning with the same objective, and we leave it as future work. 
Although we do not need to track the gradients during all sampling steps, we still need to run $T'$ inference steps to collect the sequence, which is inevitably slower than GAN. 

\begin{table}[h]
\centering
\vskip 0.08in
\scalebox{0.9}{
\begin{tabular}{lcc}
\toprule
\textbf{Method}  &\textbf{CIFAR-10} ($32\times32$) & \textbf{CelebA} ($64\times64$)  \\
 \hline
 DDPM& 34.76 &36.69 \\
 FastDPM & 29.43 & 28.98\\
Analytic-DPM & 22.94 & 28.99\\
SN-DDPM & 16.33& 20.60\\
SFT-PG (B) & \textbf{2.28} & \textbf{2.01}\\
 \bottomrule
\end{tabular}
}
\vspace{0.1in}
\caption{FID ($\downarrow$) on CIFAR-10 and CelebA, $T'=10$ for all methods. Our fine-tuning produces comparable results with the full-step pretrained models (FID = 3.03 for CIFAR-10, and FID = 3.26 for CelebA, $T=1000$).}
\label{tab: benchmark}
\vspace{-0.3cm}
\end{table}
\section{Conclusion}

In this work, we fine-tune DDPM samplers to minimize the IPMs via policy gradient. We show performing gradient descent of stochastic Markov chains w.r.t. IPM is equivalent to policy gradient, and present a surrogate function of the IPM which sheds light on monotonic improvement conditions. Our fine-tuning improves the existing fast samplers of DDPM, achieving comparable or even higher sample quality than the full-step model on various datasets.

\section*{Acknowledgements} 
Support for this research was provided by the University of Wisconsin-Madison Office of the Vice Chancellor for Research and Graduate Education with funding from the Wisconsin Alumni Research Foundation, and NSF Award DMS-2023239.

\bibliographystyle{unsrtnat}
\bibliography{references}  

\begin{thebibliography}{48}
\providecommand{\natexlab}[1]{#1}
\providecommand{\url}[1]{\texttt{#1}}
\expandafter\ifx\csname urlstyle\endcsname\relax
  \providecommand{\doi}[1]{doi: #1}\else
  \providecommand{\doi}{doi: \begingroup \urlstyle{rm}\Url}\fi

\bibitem[Ho et~al.(2020)Ho, Jain, and Abbeel]{ho2020denoising}
Jonathan Ho, Ajay Jain, and Pieter Abbeel.
\newblock Denoising diffusion probabilistic models.
\newblock \emph{Advances in Neural Information Processing Systems},
  33:\penalty0 6840--6851, 2020.

\bibitem[Nichol and Dhariwal(2021)]{nichol2021improved}
Alexander~Quinn Nichol and Prafulla Dhariwal.
\newblock Improved denoising diffusion probabilistic models.
\newblock In \emph{International Conference on Machine Learning}, pages
  8162--8171. PMLR, 2021.

\bibitem[Dhariwal and Nichol(2021)]{dhariwal2021diffusion}
Prafulla Dhariwal and Alexander Nichol.
\newblock Diffusion models beat gans on image synthesis.
\newblock \emph{Advances in Neural Information Processing Systems},
  34:\penalty0 8780--8794, 2021.

\bibitem[Goodfellow et~al.(2014)Goodfellow, Pouget-Abadie, Mirza, Xu,
  Warde-Farley, Ozair, Courville, and Bengio]{gan}
Ian Goodfellow, Jean Pouget-Abadie, Mehdi Mirza, Bing Xu, David Warde-Farley,
  Sherjil Ozair, Aaron Courville, and Yoshua Bengio.
\newblock Generative adversarial nets.
\newblock In Z.~Ghahramani, M.~Welling, C.~Cortes, N.~Lawrence, and K.Q.
  Weinberger, editors, \emph{Advances in Neural Information Processing
  Systems}, volume~27, 2014.

\bibitem[Kong and Ping(2021)]{kong2021fast}
Zhifeng Kong and Wei Ping.
\newblock On fast sampling of diffusion probabilistic models.
\newblock \emph{arXiv preprint arXiv:2106.00132}, 2021.

\bibitem[San-Roman et~al.(2021)San-Roman, Nachmani, and Wolf]{san2021noise}
Robin San-Roman, Eliya Nachmani, and Lior Wolf.
\newblock Noise estimation for generative diffusion models.
\newblock \emph{arXiv preprint arXiv:2104.02600}, 2021.

\bibitem[Lam et~al.(2021)Lam, Wang, Huang, Su, and Yu]{lam2021bilateral}
Max~WY Lam, Jun Wang, Rongjie Huang, Dan Su, and Dong Yu.
\newblock Bilateral denoising diffusion models.
\newblock \emph{arXiv preprint arXiv:2108.11514}, 2021.

\bibitem[Watson et~al.(2021{\natexlab{a}})Watson, Ho, Norouzi, and
  Chan]{watson2021learningto}
Daniel Watson, Jonathan Ho, Mohammad Norouzi, and William Chan.
\newblock Learning to efficiently sample from diffusion probabilistic models.
\newblock \emph{arXiv preprint arXiv:2106.03802}, 2021{\natexlab{a}}.

\bibitem[Jolicoeur-Martineau et~al.(2021)Jolicoeur-Martineau, Li,
  Pich{\'e}-Taillefer, Kachman, and Mitliagkas]{jolicoeur2021gotta}
Alexia Jolicoeur-Martineau, Ke~Li, R{\'e}mi Pich{\'e}-Taillefer, Tal Kachman,
  and Ioannis Mitliagkas.
\newblock Gotta go fast when generating data with score-based models.
\newblock \emph{arXiv preprint arXiv:2105.14080}, 2021.

\bibitem[Bao et~al.(2021)Bao, Li, Zhu, and Zhang]{bao2021analytic}
Fan Bao, Chongxuan Li, Jun Zhu, and Bo~Zhang.
\newblock Analytic-dpm: an analytic estimate of the optimal reverse variance in
  diffusion probabilistic models.
\newblock In \emph{International Conference on Learning Representations}, 2021.

\bibitem[Bao et~al.(2022)Bao, Li, Sun, Zhu, and Zhang]{bao2022estimating}
Fan Bao, Chongxuan Li, Jiacheng Sun, Jun Zhu, and Bo~Zhang.
\newblock Estimating the optimal covariance with imperfect mean in diffusion
  probabilistic models.
\newblock In \emph{International Conference on Machine Learning}, pages
  1555--1584. PMLR, 2022.

\bibitem[Xiao et~al.(2021)Xiao, Kreis, and Vahdat]{xiao2021tackling}
Zhisheng Xiao, Karsten Kreis, and Arash Vahdat.
\newblock Tackling the generative learning trilemma with denoising diffusion
  gans.
\newblock In \emph{International Conference on Learning Representations}, 2021.

\bibitem[Song et~al.(2020{\natexlab{a}})Song, Meng, and
  Ermon]{song2020denoising}
Jiaming Song, Chenlin Meng, and Stefano Ermon.
\newblock Denoising diffusion implicit models.
\newblock In \emph{International Conference on Learning Representations},
  2020{\natexlab{a}}.

\bibitem[Hussein et~al.(2017)Hussein, Gaber, Elyan, and
  Jayne]{hussein2017imitation}
Ahmed Hussein, Mohamed~Medhat Gaber, Eyad Elyan, and Chrisina Jayne.
\newblock Imitation learning: A survey of learning methods.
\newblock \emph{ACM Computing Surveys (CSUR)}, 50\penalty0 (2):\penalty0 1--35,
  2017.

\bibitem[Torabi et~al.(2018)Torabi, Warnell, and Stone]{torabi2018behavioral}
Faraz Torabi, Garrett Warnell, and Peter Stone.
\newblock Behavioral cloning from observation.
\newblock In \emph{Proceedings of the 27th International Joint Conference on
  Artificial Intelligence}, pages 4950--4957, 2018.

\bibitem[Song et~al.(2020{\natexlab{b}})Song, Sohl-Dickstein, Kingma, Kumar,
  Ermon, and Poole]{song2020score}
Yang Song, Jascha Sohl-Dickstein, Diederik~P Kingma, Abhishek Kumar, Stefano
  Ermon, and Ben Poole.
\newblock Score-based generative modeling through stochastic differential
  equations.
\newblock In \emph{International Conference on Learning Representations},
  2020{\natexlab{b}}.

\bibitem[Kwon et~al.(2022)Kwon, Fan, and Lee]{kwon2022scorebased}
Dohyun Kwon, Ying Fan, and Kangwook Lee.
\newblock Score-based generative modeling secretly minimizes the wasserstein
  distance.
\newblock In \emph{Advances in Neural Information Processing Systems}, 2022.

\bibitem[Salimans and Ho(2021)]{salimans2021progressive}
Tim Salimans and Jonathan Ho.
\newblock Progressive distillation for fast sampling of diffusion models.
\newblock In \emph{International Conference on Learning Representations}, 2021.

\bibitem[Rezende et~al.(2014)Rezende, Mohamed, and
  Wierstra]{rezende2014stochastic}
Danilo~Jimenez Rezende, Shakir Mohamed, and Daan Wierstra.
\newblock Stochastic backpropagation and approximate inference in deep
  generative models.
\newblock In \emph{International conference on machine learning}, pages
  1278--1286. PMLR, 2014.

\bibitem[Arjovsky et~al.(2017)Arjovsky, Chintala, and
  Bottou]{arjovsky2017wasserstein}
Martin Arjovsky, Soumith Chintala, and L{\'e}on Bottou.
\newblock Wasserstein generative adversarial networks.
\newblock In \emph{International conference on machine learning}, pages
  214--223. PMLR, 2017.

\bibitem[Gulrajani et~al.(2017)Gulrajani, Ahmed, Arjovsky, Dumoulin, and
  Courville]{gulrajani2017improved}
Ishaan Gulrajani, Faruk Ahmed, Martin Arjovsky, Vincent Dumoulin, and Aaron~C
  Courville.
\newblock Improved training of wasserstein gans.
\newblock \emph{Advances in neural information processing systems}, 30, 2017.

\bibitem[Arbel et~al.(2018)Arbel, Sutherland, Bi{\'n}kowski, and
  Gretton]{arbel2018gradient}
Michael Arbel, Danica~J Sutherland, Miko{\l}aj Bi{\'n}kowski, and Arthur
  Gretton.
\newblock On gradient regularizers for mmd gans.
\newblock \emph{Advances in neural information processing systems}, 31, 2018.

\bibitem[Williams(1992)]{williams1992simple}
Ronald~J Williams.
\newblock Simple statistical gradient-following algorithms for connectionist
  reinforcement learning.
\newblock \emph{Reinforcement learning}, pages 5--32, 1992.

\bibitem[Mohamed et~al.(2020)Mohamed, Rosca, Figurnov, and
  Mnih]{mohamed2020monte}
Shakir Mohamed, Mihaela Rosca, Michael Figurnov, and Andriy Mnih.
\newblock Monte carlo gradient estimation in machine learning.
\newblock \emph{J. Mach. Learn. Res.}, 21\penalty0 (132):\penalty0 1--62, 2020.

\bibitem[Schulman et~al.(2015{\natexlab{a}})Schulman, Levine, Abbeel, Jordan,
  and Moritz]{schulman2015trust}
John Schulman, Sergey Levine, Pieter Abbeel, Michael Jordan, and Philipp
  Moritz.
\newblock Trust region policy optimization.
\newblock In \emph{International conference on machine learning}, pages
  1889--1897. PMLR, 2015{\natexlab{a}}.

\bibitem[Schulman et~al.(2015{\natexlab{b}})Schulman, Moritz, Levine, Jordan,
  and Abbeel]{schulman2015high}
John Schulman, Philipp Moritz, Sergey Levine, Michael Jordan, and Pieter
  Abbeel.
\newblock High-dimensional continuous control using generalized advantage
  estimation.
\newblock \emph{arXiv preprint arXiv:1506.02438}, 2015{\natexlab{b}}.

\bibitem[Mescheder et~al.(2018)Mescheder, Geiger, and
  Nowozin]{mescheder2018training}
Lars Mescheder, Andreas Geiger, and Sebastian Nowozin.
\newblock Which training methods for gans do actually converge?
\newblock In \emph{International conference on machine learning}, pages
  3481--3490. PMLR, 2018.

\bibitem[Yu et~al.(2017)Yu, Zhang, Wang, and Yu]{yu2017seqgan}
Lantao Yu, Weinan Zhang, Jun Wang, and Yong Yu.
\newblock Seqgan: Sequence generative adversarial nets with policy gradient.
\newblock In \emph{Proceedings of the AAAI conference on artificial
  intelligence}, volume~31, 2017.

\bibitem[Wang et~al.(2017)Wang, Yu, Zhang, Gong, Xu, Wang, Zhang, and
  Zhang]{wang2017irgan}
Jun Wang, Lantao Yu, Weinan Zhang, Yu~Gong, Yinghui Xu, Benyou Wang, Peng
  Zhang, and Dell Zhang.
\newblock Irgan: A minimax game for unifying generative and discriminative
  information retrieval models.
\newblock In \emph{Proceedings of the 40th International ACM SIGIR conference
  on Research and Development in Information Retrieval}, pages 515--524, 2017.

\bibitem[Sarmad et~al.(2019)Sarmad, Lee, and Kim]{sarmad2019rl}
Muhammad Sarmad, Hyunjoo~Jenny Lee, and Young~Min Kim.
\newblock Rl-gan-net: A reinforcement learning agent controlled gan network for
  real-time point cloud shape completion.
\newblock In \emph{Proceedings of the IEEE/CVF Conference on Computer Vision
  and Pattern Recognition}, pages 5898--5907, 2019.

\bibitem[Bai et~al.(2019)Bai, Guan, and Wang]{bai2019model}
Xueying Bai, Jian Guan, and Hongning Wang.
\newblock A model-based reinforcement learning with adversarial training for
  online recommendation.
\newblock \emph{Advances in Neural Information Processing Systems}, 32, 2019.

\bibitem[Zheng et~al.(2022)Zheng, He, Chen, and Zhou]{zheng2022truncated}
Huangjie Zheng, Pengcheng He, Weizhu Chen, and Mingyuan Zhou.
\newblock Truncated diffusion probabilistic models and diffusion-based
  adversarial auto-encoders.
\newblock \emph{arXiv preprint arXiv:2202.09671}, 2022.

\bibitem[Wang et~al.(2022)Wang, Zheng, He, Chen, and Zhou]{wang2022diffusion}
Zhendong Wang, Huangjie Zheng, Pengcheng He, Weizhu Chen, and Mingyuan Zhou.
\newblock Diffusion-gan: Training gans with diffusion.
\newblock \emph{arXiv preprint arXiv:2206.02262}, 2022.

\bibitem[Luhman and Luhman(2021)]{luhman2021knowledge}
Eric Luhman and Troy Luhman.
\newblock Knowledge distillation in iterative generative models for improved
  sampling speed.
\newblock \emph{arXiv preprint arXiv:2101.02388}, 2021.

\bibitem[Liu et~al.(2022)Liu, Ren, Lin, and Zhao]{liu2022pseudo}
Luping Liu, Yi~Ren, Zhijie Lin, and Zhou Zhao.
\newblock Pseudo numerical methods for diffusion models on manifolds.
\newblock \emph{arXiv preprint arXiv:2202.09778}, 2022.

\bibitem[Lu et~al.(2022)Lu, Zhou, Bao, Chen, Li, and Zhu]{lu2022dpm}
Cheng Lu, Yuhao Zhou, Fan Bao, Jianfei Chen, Chongxuan Li, and Jun Zhu.
\newblock Dpm-solver: A fast ode solver for diffusion probabilistic model
  sampling in around 10 steps.
\newblock \emph{arXiv preprint arXiv:2206.00927}, 2022.

\bibitem[Silver et~al.(2014)Silver, Lever, Heess, Degris, Wierstra, and
  Riedmiller]{silver2014deterministic}
David Silver, Guy Lever, Nicolas Heess, Thomas Degris, Daan Wierstra, and
  Martin Riedmiller.
\newblock Deterministic policy gradient algorithms.
\newblock In \emph{International conference on machine learning}, pages
  387--395. PMLR, 2014.

\bibitem[Watson et~al.(2021{\natexlab{b}})Watson, Chan, Ho, and
  Norouzi]{watson2021learning}
Daniel Watson, William Chan, Jonathan Ho, and Mohammad Norouzi.
\newblock Learning fast samplers for diffusion models by differentiating
  through sample quality.
\newblock In \emph{International Conference on Learning Representations},
  2021{\natexlab{b}}.

\bibitem[Bi{\'n}kowski et~al.(2018)Bi{\'n}kowski, Sutherland, Arbel, and
  Gretton]{binkowski2018demystifying}
Miko{\l}aj Bi{\'n}kowski, Danica~J Sutherland, Michael Arbel, and Arthur
  Gretton.
\newblock Demystifying mmd gans.
\newblock \emph{arXiv preprint arXiv:1801.01401}, 2018.

\bibitem[Aiello et~al.(2023)Aiello, Valsesia, and Magli]{aiello2023fast}
Emanuele Aiello, Diego Valsesia, and Enrico Magli.
\newblock Fast inference in denoising diffusion models via mmd finetuning.
\newblock \emph{arXiv preprint arXiv:2301.07969}, 2023.

\bibitem[Li et~al.(2017)Li, Chang, Cheng, Yang, and P{\'o}czos]{li2017mmd}
Chun-Liang Li, Wei-Cheng Chang, Yu~Cheng, Yiming Yang, and Barnab{\'a}s
  P{\'o}czos.
\newblock Mmd gan: Towards deeper understanding of moment matching network.
\newblock \emph{Advances in neural information processing systems}, 30, 2017.

\bibitem[Pedregosa et~al.(2011)Pedregosa, Varoquaux, Gramfort, Michel, Thirion,
  Grisel, Blondel, Prettenhofer, Weiss, Dubourg, Vanderplas, Passos,
  Cournapeau, Brucher, Perrot, and Duchesnay]{scikit-learn}
F.~Pedregosa, G.~Varoquaux, A.~Gramfort, V.~Michel, B.~Thirion, O.~Grisel,
  M.~Blondel, P.~Prettenhofer, R.~Weiss, V.~Dubourg, J.~Vanderplas, A.~Passos,
  D.~Cournapeau, M.~Brucher, M.~Perrot, and E.~Duchesnay.
\newblock Scikit-learn: Machine learning in {P}ython.
\newblock \emph{Journal of Machine Learning Research}, 12:\penalty0 2825--2830,
  2011.

\bibitem[Flamary et~al.(2021)Flamary, Courty, Gramfort, Alaya, Boisbunon,
  Chambon, Chapel, Corenflos, Fatras, Fournier, Gautheron, Gayraud, Janati,
  Rakotomamonjy, Redko, Rolet, Schutz, Seguy, Sutherland, Tavenard, Tong, and
  Vayer]{flamary2021pot}
R{\'e}mi Flamary, Nicolas Courty, Alexandre Gramfort, Mokhtar~Z. Alaya,
  Aur{\'e}lie Boisbunon, Stanislas Chambon, Laetitia Chapel, Adrien Corenflos,
  Kilian Fatras, Nemo Fournier, L{\'e}o Gautheron, Nathalie~T.H. Gayraud,
  Hicham Janati, Alain Rakotomamonjy, Ievgen Redko, Antoine Rolet, Antony
  Schutz, Vivien Seguy, Danica~J. Sutherland, Romain Tavenard, Alexander Tong,
  and Titouan Vayer.
\newblock Pot: Python optimal transport.
\newblock \emph{Journal of Machine Learning Research}, 22\penalty0
  (78):\penalty0 1--8, 2021.

\bibitem[LeCun et~al.(1998)LeCun, Bottou, Bengio, and
  Haffner]{lecun1998gradient}
Yann LeCun, L{\'e}on Bottou, Yoshua Bengio, and Patrick Haffner.
\newblock Gradient-based learning applied to document recognition.
\newblock \emph{Proceedings of the IEEE}, 86\penalty0 (11):\penalty0
  2278--2324, 1998.

\bibitem[Krizhevsky et~al.(2009)Krizhevsky, Hinton,
  et~al.]{krizhevsky2009learning}
Alex Krizhevsky, Geoffrey Hinton, et~al.
\newblock Learning multiple layers of features from tiny images.
\newblock 2009.

\bibitem[Liu et~al.(2015)Liu, Luo, Wang, and Tang]{liu2015deep}
Ziwei Liu, Ping Luo, Xiaogang Wang, and Xiaoou Tang.
\newblock Deep learning face attributes in the wild.
\newblock In \emph{Proceedings of the IEEE international conference on computer
  vision}, pages 3730--3738, 2015.

\bibitem[Heusel et~al.(2017)Heusel, Ramsauer, Unterthiner, Nessler, and
  Hochreiter]{heusel2017gans}
Martin Heusel, Hubert Ramsauer, Thomas Unterthiner, Bernhard Nessler, and Sepp
  Hochreiter.
\newblock Gans trained by a two time-scale update rule converge to a local nash
  equilibrium.
\newblock \emph{Advances in neural information processing systems}, 30, 2017.

\bibitem[Kingma and Ba(2014)]{kingma2014adam}
Diederik~P Kingma and Jimmy Ba.
\newblock Adam: A method for stochastic optimization.
\newblock \emph{arXiv preprint arXiv:1412.6980}, 2014.

\end{thebibliography}
\newpage
\appendix
\onecolumn
\section{Visualization: Effect of Shortcut Fine-Tuning}
\label{app: vis}

\begin{figure}[h]
\centering
\begin{subfigure}[b]{0.45\linewidth}
    \centering
\includegraphics[width=\linewidth]{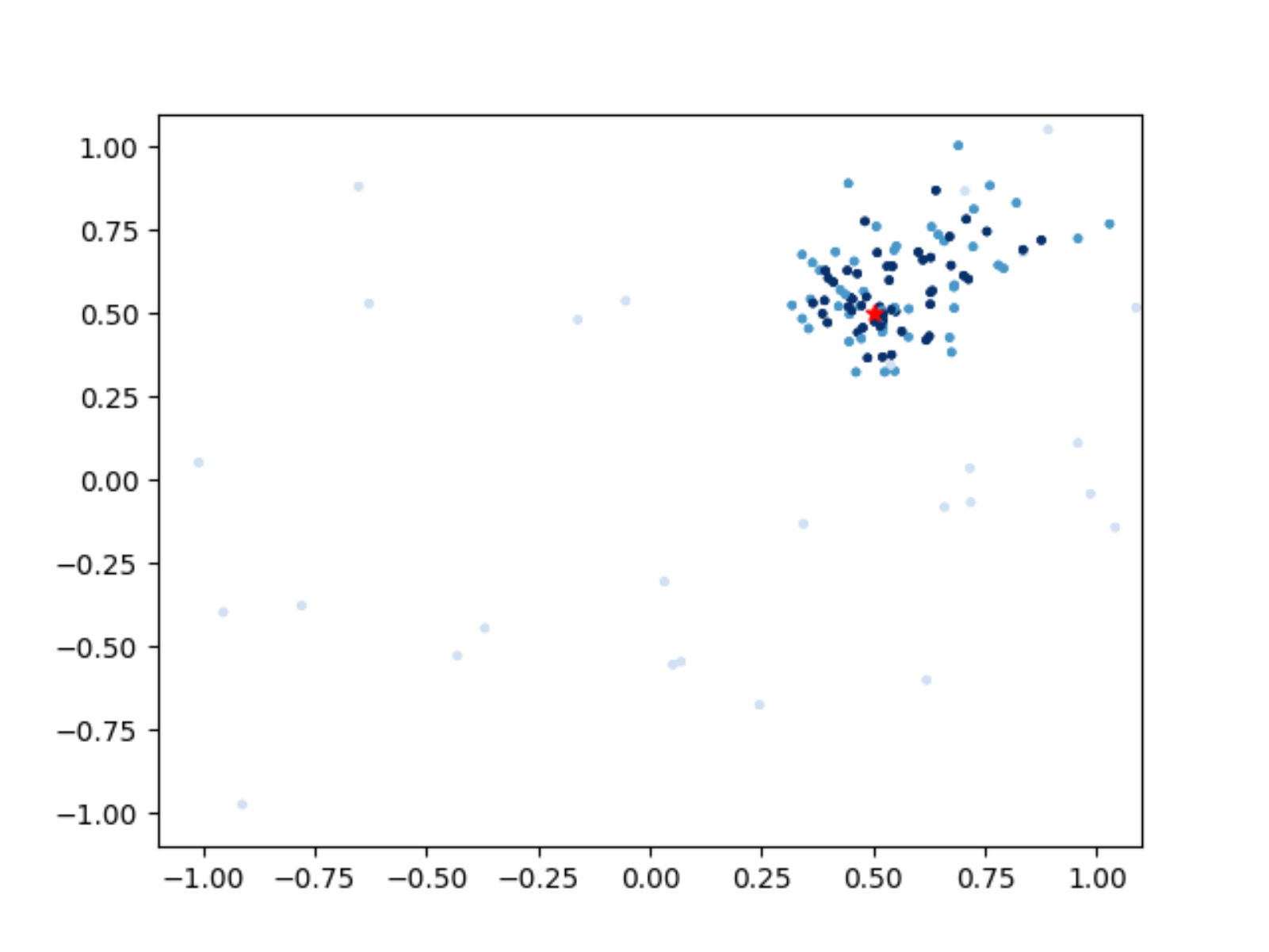}
\caption{Sampling steps before fine-tuning from DDPM}
\label{fig: vis}
\end{subfigure}
\hfill
\begin{subfigure}[b]{0.45\linewidth}
    \centering
\includegraphics[width=\linewidth]{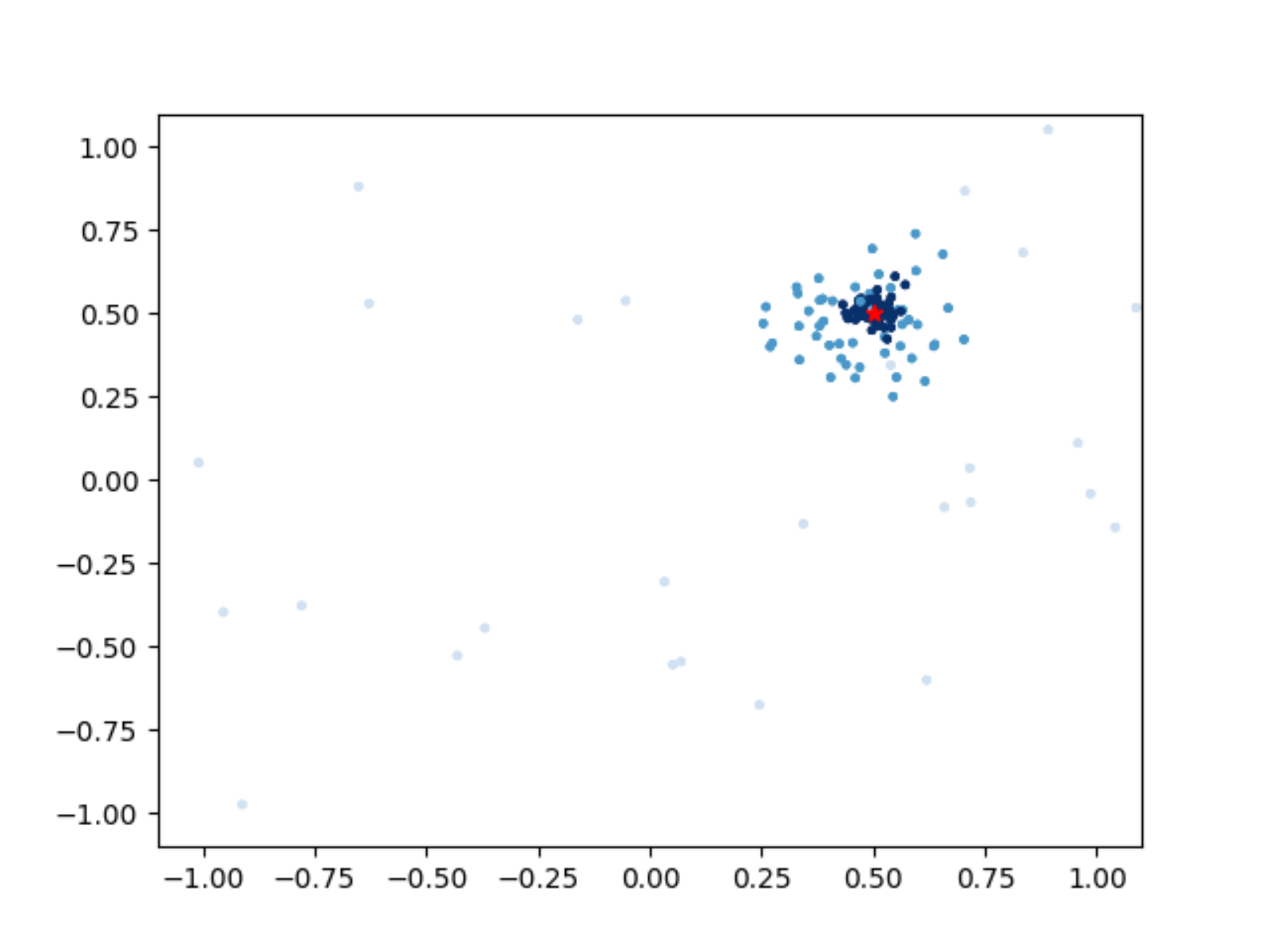}
\caption{Sampling steps after fine-tuning}
\label{fig: vis_new}
\end{subfigure}
\caption{Visualization of the sampling path before~(\ref{fig: vis}) and after short-cut fine-tuning~(\ref{fig: vis_new}).}
\label{fig: visualization}
\end{figure}

We provide visualizations of the complete sampling chain before and after fine-tuning in Fig~\ref{fig: visualization}. We generate 50 data points using the same random seed for DDPM and our fine-tuned model, trained on the same Gaussian cluster centered at the red spot $(0.5, 0.5)$ with a standard deviation of 0.01 in each dimension, $T=2$. The whole sampling path is visualized where different steps are marked with different intensities of the color: data points with the darkest color are finally generated. As shown in Fig~\ref{fig: visualization}, our fine-tuning does find a "shortcut" path to the final distribution.
\section{Illustration of Sub-sampling with $T'\ll T$ in DDPM}
\label{app: subsampling}

\begin{figure}[h!]
\centering
\centering
\includegraphics[width=0.5\linewidth]{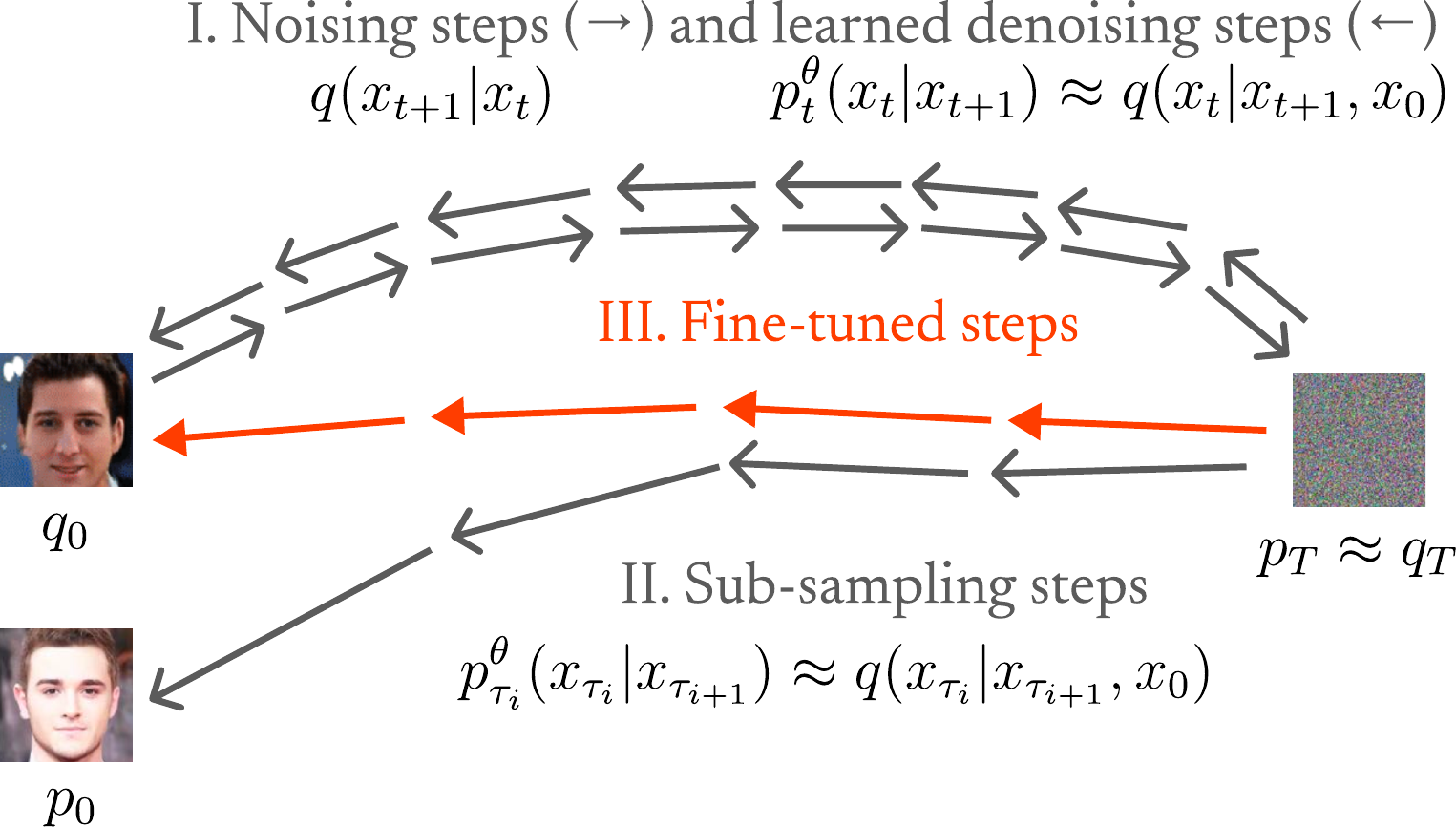}
\caption{
\textcolor{black}{When $T$ is large but we sub-sample with $T' \ll T$ cannot approximate the backward process accurately when each step is Gaussian as discussed in Case 2, Section~\ref{sec: motivation}. In this case, shortcut fine-tuning can also solve the issue by directly minimizing the IPM as an objective function.}}
\label{fig:fastsampling}
\end{figure}
\section{Towards Monotonic Improvement}

\label{app: mono}
Here we present detailed proof of Theorem~\ref{mono}.
For simplicity, we denote $p_0^\theta$ as $p_\theta$, $q_0$ as $q$, and $z\in \mathbb{R}^d$ to replace $x_0$ as a variable in our sample space. 

Recall the generated distribution: $p_{\theta}$.
Given target distribution $q$, the objective function is:
\begin{equation}
\min_{\theta}\max_{\alpha \in \mathcal{A}}g(p_{\theta}, f_\alpha, q),
\end{equation}
where $g(p_{\theta}, f_\alpha, q) = \int (p_{\theta}(z) - q(z))f_{\alpha}(z)dz$.


Recall $\Phi(p_{\theta},q) = \underset{\alpha \in \mathcal{A}}{\max}\int (p_{\theta}(z) - q(z))f_{\alpha}(z)dz  =\int (p_{\theta}(z) - q(z))f_{\alpha^*(p_\theta, q)}(z)dz$.

Assume that $g(p_{\theta}, f_\alpha, q)$ is Lipschitz w.r.t. $\theta$, given $q$ and $\alpha\in\mathcal{A}$. Our goal is to show that there exists $l\geq 0$ s.t.:
\begin{equation}
     \Phi(p_{\theta'}, q) \leq g(p_{\theta'},f_{\alpha^*(p_\theta,q)}, q) + 2l||\theta-\theta'||,
\end{equation}
where the equality is achieved when $\theta = \theta'$. 
 
If the above inequality holds, $L_{\theta}(\theta') = g(p_{\theta'},f_{\alpha^*(p_\theta,q)}, q) + 2l||\theta-\theta'||$ can be a surrogate function of $\Phi(p_{\theta'}, q)$:  $\Phi(p_{\theta'}, q) -\Phi(p_\theta, q) \leq L_{\theta}(\theta')- L_{\theta}(\theta)$, $L_{\theta}(\theta) = \Phi(p_\theta, q)$, which means $\theta'$ that can improve $L_{\theta}(\theta')$ is also guaranteed to get improvement on $\Phi(p_{\theta'}, q)$. 

\begin{proof}
Consider
\begin{equation}
\begin{split}
&\mathrel{\phantom{=}}\Phi(p_{\theta'}, q) - \Phi(p_{\theta}, q)\\
& =\int (p_{\theta'} (z)- q(z))f_{\alpha^*(p_{\theta'}, q)}(z)dz-\int (p_{\theta} (z)- q(z))f_{\alpha^*(p_\theta, q)}(z)dz\\
& =\int (p_{\theta'} (z)- q(z))f_{\alpha^*(p_{\theta'}, q)}(z)dz -\int (p_{\theta'} (z)- q(z))f_{\alpha^*(p_\theta, q)}(z)dz \\
& \mathrel{\phantom{=}}+\int (p_{\theta'} (z)- q(z))f_{\alpha^*(p_\theta, q)}(z)dz- \int (p_{\theta} (z)- q(z))f_{\alpha^*(p_\theta, q)}(z)dz\\
& = \int (p_{\theta'} (z)- q(z))(f_{\alpha^*(p_{\theta'}, q)}(z)-f_{\alpha^*(p_\theta, q)}(z))dz + \int (p_{\theta'} (z)- p_{\theta} (z))f_{\alpha^*(p_\theta, q)}(z)dz\\
& =  \int (q(z)-p_{\theta'} (z))(f_{\alpha^*(p_\theta, q)}(z)-f_{\alpha^*(p_{\theta'}, q)}(z))dz + \int (p_{\theta'} (z)- p_{\theta} (z))f_{\alpha^*(p_\theta, q)}(z)dz.\\
\end{split}
\end{equation}

We have

\begin{equation}
\begin{split}
&\mathrel{\phantom{=}}\int (q(z)-p_{\theta'} (z))(f_{\alpha^*(p_\theta, q)}(z)-f_{\alpha^*(p_{\theta'}, q)}(z))dz\\
&=\int (p_{\theta}(z)-p_{\theta'} (z))(f_{\alpha^*(p_\theta, q)}(z)-f_{\alpha^*(p_{\theta'}, q)}(z))dz-\int (p_{\theta}(z)-q(z))(f_{\alpha^*(p_\theta, q)}(z)-f_{\alpha^*(p_{\theta'}, q)}(z))dz\\
&\leq \int (p_{\theta}(z)-p_{\theta'} (z))(f_{\alpha^*(p_\theta, q)}(z)-f_{\alpha^*(p_{\theta'}, q)}(z))dz,
\end{split}
\end{equation}
where the last inequality comes from the definition: $\alpha^*(p_\theta, q) = \underset{\alpha \in \mathcal{A}}{\argmax} \int (p_{\theta}(z)-q(z))f_{\alpha}(z)$.

So  
\begin{equation}
\begin{split}
&\mathrel{\phantom{=}}\Phi(p_{\theta'},q) - \Phi(p_{\theta},q) \\
&= \int (p_{\theta'} (z)- p_{\theta} (z))f_{\alpha^*(p_\theta, q)}(z)dz + \int (q(z)-p_{\theta'} (z))(f_{\alpha^*(p_\theta, q)}(z)-f_{\alpha^*(p_{\theta'}, q)}(z))dz \\
&\leq g(p_{\theta'}, f_{\alpha^*(p_\theta, q)}, q) - g(p_{\theta}, f_{\alpha^*(p_\theta, q)},q) + \int (p_{\theta}(z)-p_{\theta'}(z))(f_{\alpha^*(p_\theta, q)}(z)-f_{\alpha^*(p_{\theta'}, q)}(z))dz\\
&\leq g(p_{\theta'}, f_{\alpha^*(p_\theta, q)},q) - g(p_{\theta}, f_{\alpha^*(p_\theta, q)},q) + 2l||\theta-\theta'||,\\
\end{split}
\end{equation}

where the last inequality comes from the Lipschitz assumption of $g(p_{\theta}, f_{\alpha(p_\theta, q)},q)$ given $\alpha^*(p_\theta, q)$ and $\alpha^*(p_{\theta'}, q)$. Recall that $\Phi(p_{\theta},q) = g(p_{\theta}, f_{\alpha^*(p_\theta, q)},q)$, the proof is then complete.
\end{proof}



Consider the optimization objective: $\text{minimize}_{\theta'} L_{\theta}(\theta')$. Using the Lagrange multiplier, we can convert the problem to a constrained optimization problem: 
    
\begin{equation}
\begin{split}
&\underset{\theta'}{\text{minimize}} \quad g(p_{\theta'},f_{\alpha^*(p_\theta, q)}, q) \\
& \text{s.t.} \quad||\theta' - \theta||\leq \delta \\
\end{split}
\end{equation}
where $\delta >0$. The constraint is a convex set and the projection to the set is easy to compute via norm regularization, as we discussed in Section~\ref{sec: mono}. Intuitively, it means that as long as we only optimize in the neighborhood of the current generator $\theta'$, we can treat $g(p_{\theta'},f_{\alpha^*(p_\theta, q)}, q)$ as an  approximation of $\Phi(p_{\theta'}, q)$ during gradient updates.

\section{Baseline Function for Variance Reduction}
\label{app: baseline}
Here we present the derivation of Eq~(\ref{eq: baseline}), which is very similar to \citet{schulman2015high}.

To show

\begin{equation}
\underset{p_{x_{0:T}}^\theta}{\mathbb{E}}\left[f_{\alpha}(x_0) \sum_{t=0}^{T-1} \nabla_{\theta}\log p_t^{\theta}(x_t|x_{t+1}) \right]=\underset{p_{x_{0:T}}^\theta}{\mathbb{E}}\left[\sum_{t=0}^{T-1}(f_{\alpha}(x_0) - V_{t+1}^\omega(x_{t+1}))\nabla_{\theta}\log p_t^{\theta}(x_t|x_{t+1})  \right],
\end{equation}
we only need to show 

\begin{equation}
    \underset{p_{x_{0:T}}^\theta}{\mathbb{E}}\left[V_{t+1}^\omega(x_{t+1})\nabla_{\theta}\log p_t^{\theta}(x_t|x_{t+1})  \right] =0.
\end{equation}
Note that
\begin{equation}
\begin{aligned}
&\mathrel{\phantom{=}}\underset{p_{x_{0:T}}^\theta}{\mathbb{E}}\left[V_{t+1}^\omega(x_{t+1})\nabla_{\theta}\log p_t^{\theta}(x_t|x_{t+1})  \right] \\
&=\underset{p_{x_{t+1:T}}^\theta}{\mathbb{E}}\left[\underset{p_{x_{0:t}}^\theta}{\mathbb{E}}\left[V_{t+1}^\omega(x_{t+1})\nabla_{\theta}\log p_t^{\theta}(x_t|x_{t+1}) |x_{t+1:T}\right] \right]\\
&=\underset{p_{x_{t+1:T}}^\theta}{\mathbb{E}}\left[\underset{p_{x_{t}}^\theta}{\mathbb{E}}\left[V_{t+1}^\omega(x_{t+1})\nabla_{\theta}\log p_t^{\theta}(x_t|x_{t+1})|x_{t+1:T}\right] \right],\\
\end{aligned}
\end{equation}

where $\underset{p_{x_{t}}^\theta}{\mathbb{E}}\left[V_{t+1}^\omega(x_{t+1})\nabla_{\theta}\log p_t^{\theta}(x_t|x_{t+1}) |x_{t+1:T}\right]=0$ when $p_t^{\theta}(x_t|x_{t+1})$ and $\nabla_{\theta}p_t^{\theta}(x_t|x_{t+1})$ are continuous:  

\begin{equation}
\begin{aligned}
&\mathrel{\phantom{=}}\underset{p_{x_{t}}^\theta}{\mathbb{E}}\left[V_{t+1}^\omega(x_{t+1})\nabla_{\theta}\log p_t^{\theta}(x_t|x_{t+1}) |x_{t+1:T}\right]\\  
&=V_{t+1}^\omega(x_{t+1}) \int p_{x_t}^\theta(x_t)\nabla_{\theta}\log p_t^{\theta}(x_t|x_{t+1})dx_t\\
&=V_{t+1}^\omega(x_{t+1})  \int p_{x_t}^\theta(x_t)\nabla_{\theta}\log p_t^{\theta}(x_t|x_{t+1})dx_t\\
&=V_{t+1}^\omega(x_{t+1})  \int\nabla_{\theta}p_t^{\theta}(x_t|x_{t+1})dx_t\\
&=V_{t+1}^\omega(x_{t+1}) \nabla_{\theta} \int p_t^{\theta}(x_t|x_{t+1})dx_t\\
&=0.
\end{aligned}
\end{equation}

\section{Comparison with DDIM Sampling}
\label{app: ddim}
We present a comparison with DDIM sampling methods on CIFAR 10 benchmark as below. Methods marked with * require additional model training, and NFE is the number of sampling steps (number of score function evaluations). All methods are based on the same pretrained DDPM model with $T = 1000$.
\begin{table}[H]
\centering

\begin{tabular}{lllllll}
\toprule
Method (DDPM, stochastic)  & NFE& FID && Method (DDIM, deterministic) & NFE & FID\\
 \hline
DDPM          & 10  & 34.76 &  & DDIM                         & 10  & 17.33 \\ 
SN-DDPM       & 10  & 16.33 &  & DPM-solver                   & 10  & 4.70  \\ 
SFT-PG* & 10  & 2.28  &  &                               &     &       \\ 
SFT-PG* & 8   & 2.64  &  & Progressive distillation* & 8   & 2.57  \\ 
 \bottomrule
\end{tabular}

\caption{Comparison with DDIM sampling methods which is deterministic given the initial noise.}
\end{table}
We can observe that SFT-PG with NFE=10 produces the best FID, and SFT-PG with NFE=8 is comparable to progressive distillation with the same NFE. Our method is orthogonal to other fast sampling methods like distillation. We also note that our fine-tuning is more computationally efficient than progressive distillation: For example, for CIFAR10, progressive distillation takes about a day using 8 TPUv4 chips, while our method takes about 6h using 4 RTX 2080Ti, and the original DDPM training takes 10.6h using TPU v3.8. Besides, since we use a fixed small learning rate during training (1e-6), it is also possible to further accelerate our training by choosing appropriate learning rate schedules. 
\section{Experimental Details}
\label{app: exp}

Here we provide more details for our fine-tuning settings for reproducibility.
\subsection{Experiments on Toy Datasets}
\paragraph{Training sets.} For 2D toy datasets, each training set contains 10K samples.  

\paragraph{Model architecture.}The generator we adopt is a 4-layer MLP with 128 hidden units and soft-plus activations. The critic and the baseline function we use are 3-layer MLPs with 128 hidden units and ReLU activations. 

\paragraph{Training details.}For optimizers, we use Adam~\citep{kingma2014adam} with $\text{lr} = 5\times 10^{-5}$ for the generator, and $\text{lr} = 1\times 10^{-3}$ for both the critic and baseline functions. Pretraining for DDPM is conducted for 2000 epochs for $T=10,100,1000$ respectively. Both pretraining and fine-tuning use batch size 64 and we train 300 epochs for fine-tuning.

\subsection{Experiments on Image Datasets} 
\label{app: }

\paragraph{Training sets.}
We use 60K training samples from MNIST, 50K training samples from CIFAR-10, and 162K samples from CelebA.

\paragraph{Model architecture.}
For model architecture, we use U-Net as the generative model as \citet{ho2020denoising}. For the critic, we adopt 3 convolutional layers with kernel size = 4, stride = 2, padding = 1 for downsampling,  followed by 1 final convolutional layer with kernel size = 4, stride = 1, padding = 0, and then take the average of the final output. The numbers of output channels are 256,512,1024,1 for each layer, with Leaky ReLU (slope = 0.2) as activation. For the baseline function, we use a 4-layer MLP with timestep embeddings. The numbers of hidden units are 1024, 1024, 256, and the output dimension is 1.

\paragraph{Training details.}
For MNIST, we train a DDPM with  $T=10$ steps for 100 epochs to convergence as a pretrained model. For CIFAR-10 and CelebA, we use the pretrained model in \citet{ho2020denoising} and \citet{song2020denoising} respectively with $T=1000$, and use the sampling schedules calculated by FastDPM~\citep{kong2021fast} with VAR approximation and DDPM sampling schedule as initialization for our fine-tuning. We found that rescaling the pixel values to [0,1] is a default choice in FastDPM, but it hurts the training if we put the rescaled images directly into the critic, so we remove the rescaling part during our fine-tuning. For optimizers, we use Adam with $\text{lr} = 1\times 10^{-6}$ for the generator, and $\text{lr} = 1\times 10^{-4}$ for both the critic and baseline functions. We found that smaller learning rates help the stability of training, which is compliant with the theoretical result in Section~\ref{mono}. For MNIST and CIFAR-10, we train 100 epochs with batch size = 128. For CelebA we trained 100 epochs with batch size = 64. 

\paragraph{More generated samples.} We present generated samples from the initialized FastDPM and our fine-tuned model respectively using the same random seed to show the effect of our fine-tuning in Fig~\ref{fig: more_cifar} and Fig~\ref{fig: more_celeba}. We notice that some of the images generated by our fine-tuned model are similar to images at initialization but with much richer colors and more details, and there are also some cases that the images after fine-tuning look very different than that from initialization.

\begin{figure}[h!]
\centering
\includegraphics[width=0.30\linewidth]{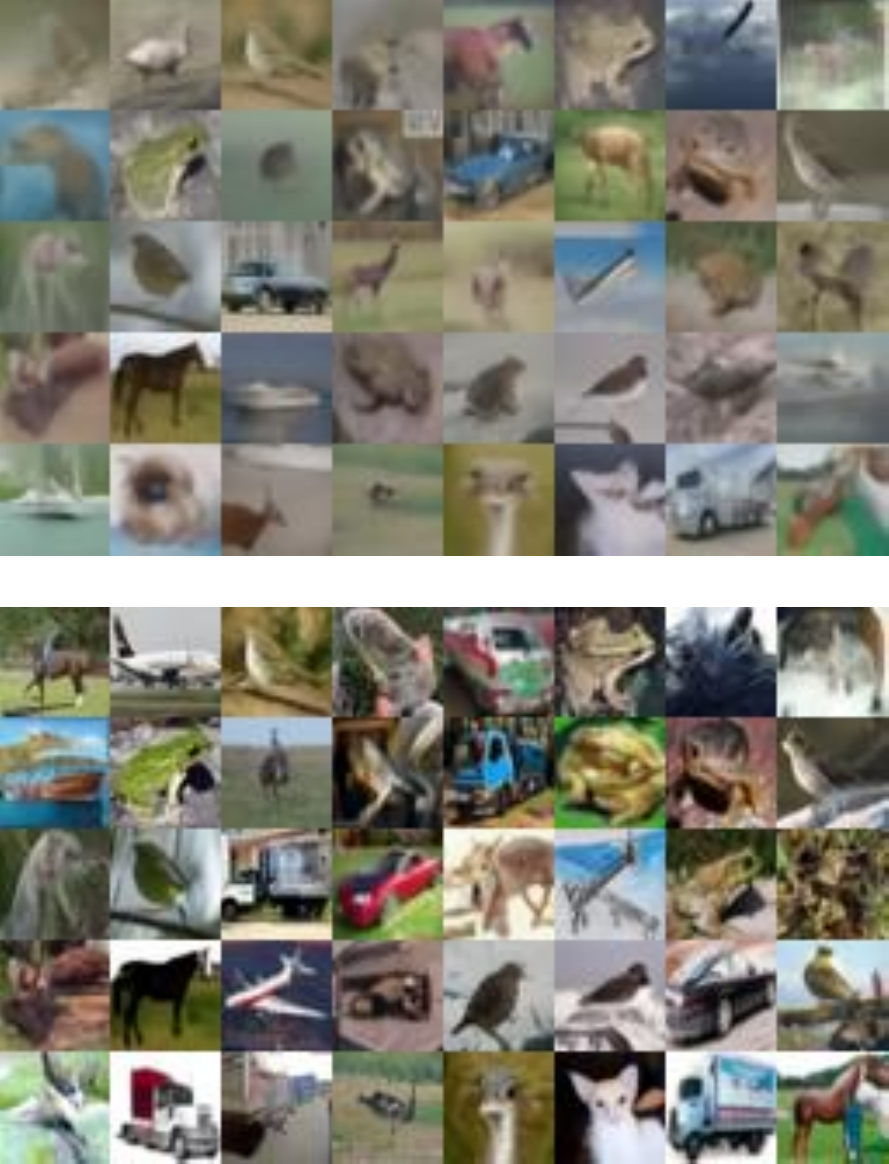}
\caption{
Images generated from FastDPM as initialization (on the top) and from the fine-tuned model (on the bottom), generated using the same seed, trained on CIFAR-10.}
\label{fig: more_cifar}
\end{figure}

\begin{figure}[h!]
\centering
\includegraphics[width=0.45\linewidth]{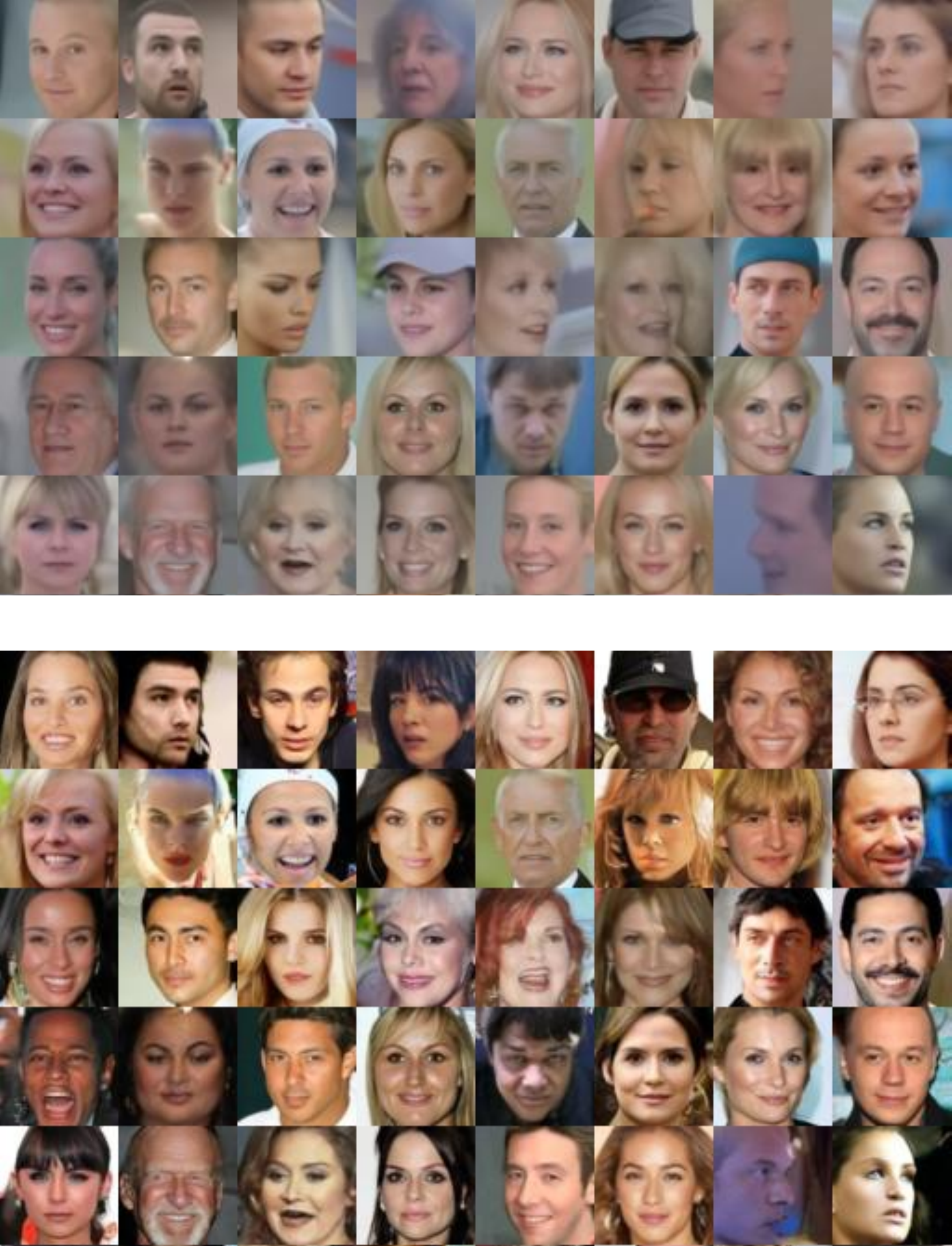}
\caption{Images generated from FastDPM as initialization (on the top) and from the fine-tuned model (on the bottom), generated using the same seed, trained on CelebA.}
\label{fig: more_celeba}
\end{figure}



\end{document}